 \documentclass[accepted,startpage=1]{uai2023} 

\usepackage[british]{babel}

\usepackage{natbib} 
\usepackage{mathtools} 
\usepackage{booktabs} 
\usepackage{tikz} 

\usepackage{amssymb}
\usepackage{MnSymbol}
\usepackage{algorithm}
\usepackage{algorithmic}
\usepackage{ulem}

\newcommand{\BlackBox}{\rule{1.5ex}{1.5ex}}  
\newenvironment{proof}{\par\noindent{\bf Proof\ }}{\hfill\BlackBox\\[2mm]}
\newtheorem{dfn}{Definition}
\newtheorem{thm}{Theorem}
\newtheorem{prp}{Proposition}		

\newtheorem{lem}[prp]{Lemma}
\newtheorem{exm}{Example}

\newcommand{\aet}{\mathrel{\leftarrow}}
\newcommand{\aea}{\mathrel{\leftrightarrow}}
\newcommand{\aem}{\mathrel{{\leftarrow\mkern-11mu\ast}}}

\newcommand{\mea}{\mathrel{{\ast\mkern-11mu\to}}}
\newcommand{\met}{\mathrel{{\ast\mkern-11mu\relbar\mkern-9mu\relbar}}}
\newcommand{\tem}{\mathrel{{\relbar\mkern-11mu\relbar\mkern-11mu\ast}}}
\newcommand{\mem}{\mathrel{{\ast\mkern-11mu\relbar\mkern-9mu\relbar\mkern-11mu\ast}}}
\newcommand{\cec}{\mathrel{{\circ\mkern-8mu\relbar\mkern-8mu\circ}}}

\newcommand{\cem}{\mathrel{{\circ\mkern-8mu\relbar\mkern-9mu\relbar\mkern-11mu\ast}}}
\newcommand{\mec}{\mathrel{{\ast\mkern-10mu\relbar\mkern-9mu\relbar\mkern-8mu\circ}}}

\newcommand{\tet}{\mathrel{{\relbar\mkern-9mu\relbar}}}

\newcommand{\tea}{\mathrel{\rightarrow}}

\newcommand{\bfA}{\mathbf{A}}
\newcommand{\bfB}{\mathbf{B}}
\newcommand{\bfC}{\mathbf{C}}

\newcommand{\bfE}{\mathbf{E}}

\newcommand{\bfO}{\mathbf{O}}

\newcommand{\bfR}{\mathbf{R}}
\newcommand{\bfS}{\mathbf{S}}

\newcommand{\bfV}{\mathbf{V}}
\newcommand{\bfW}{\mathbf{W}}
\newcommand{\bfX}{\mathbf{X}}
\newcommand{\bfY}{\mathbf{Y}}
\newcommand{\bfZ}{\mathbf{Z}}

\newcommand{\bfv}{\mathbf{v}}
\newcommand{\bfw}{\mathbf{w}}
\newcommand{\bfx}{\mathbf{x}}

\newcommand{\G}{\mathcal{G}}

\newcommand{\M}{\mathcal{M}}

\newcommand{\cP}{\mathcal{P}}

\newcommand{\seq}[1]{\langle{#1}\rangle}
\newcommand{\pa}{\; \mathrm{pa}}			

\newcommand{\bff}{\mathbf{f}}
\newcommand{\bfg}{\mathbf{g}}

\newcommand{\dom}[1]{{\mathcal{X}_{#1}}}
\newcommand{\rv}[1]{\mathfrak{X}_{#1}}
\renewcommand{\pa}{\mathrm{pa}}
\newcommand{\an}{\mathrm{an}}

\newcommand{\RN}{\mathbb{R}}
\newcommand{\Prb}{\mathbb{P}}

\title{Establishing Markov Equivalence in Cyclic Directed Graphs}

\author[1]{Tom~Claassen}
\author[2]{\href{mailto:<j.m.mooij@uva.nl>}{Joris~M.~Mooij}{}}
\affil[1]{%
    Institute for Computing and Information Sciences\\
    Radboud University\\
    Nijmegen, Netherlands
}
\affil[2]{%
    Korteweg-deVries Institute\\
    University of Amsterdam\\
    Amsterdam, Netherlands
}
  
\begin{document}
\maketitle

\begin{abstract}
  We present a new, efficient procedure to establish Markov equivalence between directed graphs that may or may not contain cycles under the \textit{d}-separation criterion.
It is based on the Cyclic Equivalence Theorem (CET) in the seminal works on cyclic models by Thomas Richardson in the mid '90s, but now rephrased from an ancestral perspective. The resulting characterization leads to a procedure for establishing Markov equivalence between graphs that no longer requires 
tests for \textit{d}-separation, leading to a significantly reduced algorithmic complexity. The conceptually simplified characterization may help to reinvigorate theoretical research towards sound and complete cyclic discovery in the presence of latent confounders.

\textbf{This version includes a correction to rule (iv) in Theorem 1, and the subsequent adjustment in part 2 of Algorithm 2.}
\end{abstract}

\section{Introduction}\label{sec:1-Intro}
Discovering causal relations from observational and experimental data is one of the key goals in many research areas. 
Developing principled, automated causal discovery methods has been an active area of research within the machine learning community, which has resulted in a wide variety of algorithms and techniques.
Two of the main challenges here are handling the impact of unobserved confounders, and the possible presence of feedback mechanisms or cycles in the system under investigation. Both have a long history in the field: in this article we solely focus on the latter.

Building on earlier work by \cite{Spirtes1994, Spirtes1995} on (linear) cyclic directed models that obey the global directed Markov property (see section \ref{sub:graphs}, below), \cite{Richardson1996a_CCD} introduced the Cyclic Causal Discovery (CCD) algorithm that was able to infer a sound cyclic causal model from independence constraints on data. It was based on the so-called Cyclic Equivalence Theorem \citep{Richardson1997} that characterized Markov equivalence between cyclic directed graphs. 

Strangely enough, after this promising start progress in cyclic directed models slowly ground to a halt, even though many challenges remained: the CCD output was certainly not complete, and could not account for latent confounders.

In the mean time theory and methods for acyclic causal discovery took flight, where, for example \cite{Zhang2008} managed to extend FCI to a provably sound and complete algorithm under latent confounders and selection bias.

And even to this day fundamental progress continues to be made: recently several new and faster algorithms and characterizations for establishing Markov equivalence between maximal ancestral graphs (graphical independence models closed under marginalization and conditioning) have been developed \citep{HuE2020,WienobstBL2022,ClaassenB2022}, ultimately bringing it down to linear complexity for sparse graphs. However, despite a widely acknowledged need to handle feedback cycles in learning algorithms for real world causal discovery, major steps towards that goal have been few and far between.

A promising attempt to extend CCD to the case of unobserved confounders was made by \citet{Strobl2018}, but though the resulting CCI algorithm was sound, it was by no means complete, foregoing on key FCI elements like discriminating paths and selection bias, and the output was not guaranteed to uniquely identify the Markov equivalence class.

Fundamentally different approaches to cyclic causal discovery have also been developed: for example, \citet{LacerdaSRH2008} employs independent component analysis, \citet{Mooij_et_al_NIPS_11,MooijHeskes_UAI_13} proposed likelihood-based structure learning approaches for additive noise models, \cite{HyttinenEH2012} exploits experiments to build a complete model, and \cite{RothenhauslerHPM2015} builds on information from unknown shift interventions to reconstruct the underlying cyclic causal graph.

On another front, \cite{ForreM2018} showed that for \textit{nonlinear} causal models with cycles and confounders, the usual $d$-separation criterion needs to be replaced with their $\sigma$-separation criterion (see also section 3 in the supplement). More recently, \cite{MooijC2020} showed that vanilla FCI was in fact already sound and complete for these nonlinear cyclic models. However, it does not account for the peculiarities encountered when handling \textit{linear} cyclic models, as in Figure \ref{fig1TwoCycle}.

For linear or discrete cyclic causal models, $\sigma$-separation is too weak, as the stronger \textit{d}-separation may apply. Perhaps surprisingly, this significantly complicates the causal structure analysis. But even in nonlinear systems we often consider linear approximations, which means in practice we may expect to encounter similar complications there as well.
  In section 3 in the supplement we summarize some results from the literature under which cyclic causal models are known to satisfy the stronger \textit{d}-separation criterion. For the current paper it suffices to know that we focus on \textit{d}-separation equivalence between cyclic directed graphs with no unobserved confounders, which, for the important class of systems where the global directed Markov condition in combination with its corresponding faithfulness assumption holds, also implies Markov equivalence.

Part of the reason for the slow progress on cyclic models that satisfy the \textit{d}-separation criterion may be that the associated theoretical machinery developed to characterize Markov equivalence is quite imposing, which may make 
extensions towards confounders seem an overly daunting task.

In this article we find things may not be quite as bad as perhaps once feared. We show, for example, that establishing Markov equivalence between directed graphs becomes more intuitive when viewed from an \textit{ancestral} perspective, leading to a simplified characterization and an efficient algorithm that greatly speeds up identification.
Although this is of course but a small step, we hope that it may inspire renewed investigation into full-fledged cyclic causal discovery in the presence of latent confounders and selection bias.

In the rest of the article, section \ref{sec:2-CyclicGraphs} introduces the necessary tools to handle cyclic directed graphs, section \ref{sec:3-NewCET} describes an alternative, ancestral formulation of the CET, section \ref{sec:4-MarkovEq} shows how to infer a graphical characterization of the Markov equivalence class without the need for \textit{d}-separation tests, and section \ref{sec:5-ExpEval} demonstrates the remarkable efficiency of the resulting procedure compared to current state of the art.
Detailed proofs 
as well as some additional experimental results are provided in the 
supplement.

\begin{figure}[h]
  \centering
  \includegraphics[width=0.9\linewidth,page=1]{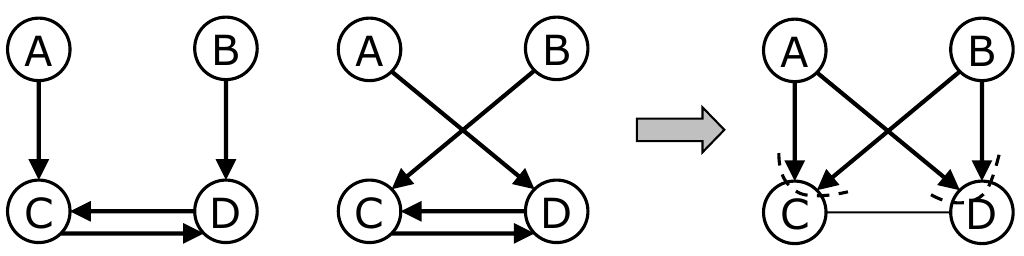}
  \caption{\small Two different cyclic graphs (left) that together form the only two members of the Markov equivalence class on the right, where the dashed lines signal two \textit{virtual} \textit{v}-structures (see $\S$\ref{sub:introCMAG}). For linear/discrete models conditioning on $C$ would make $A$ and $B$ dependent, but conditioning on $\{C,D\}$ would not.} 
  \label{fig1TwoCycle}
\end{figure}

\section{Cyclic Directed Graphs} \label{sec:2-CyclicGraphs}
In this section we start with a few standard graphical model definitions, and then continue with some perhaps less familiar terminology and results specific to cyclic graphs. 

\subsection{Graph notations and terminology} \label{sub:graphs}
Throughout this article we use capital letters for vertices/variables, boldface capitals to indicate sets, and calligraphic letters to indicate graphs or distributions.

A \textit{directed graph} (DG) $\G$ is an ordered pair $\seq{\bfV,\bfE}$, where $\bfV$ is a set of vertices (nodes), and $\bfE$ is a set of directed edges (arcs) between vertices. Two nodes in $\G$ are \textit{adjacent} if they are connected by an edge, two edges are \textit{adjacent} if they share a node. A \textit{path} in the graph $\G$ is a sequence of adjacent edges where each consecutive pair along the path is adjacent in $\G$ and each node occurs at most once, or just a single node (a \textit{trivial} path). A \textit{directed path} $X_0 \tea X_1 \tea .. \tea X_k$ is a path where each pair of consecutive nodes is connected by an arc $X_i \tea X_{i+1}$ in $\G$. A \textit{cycle} is a directed path $X_0 \tea .. \tea X_k$ together with an edge $X_k \tea X_0$. A directed graph with no cycles is called a \textit{directed acyclic graph} (DAG).
If $X \tea Y$ in $\G$ then $X$ is called a \textit{parent} of $Y$, and $Y$ a \textit{child} of $X$. Similarly, if there is a directed path from $X$ to $Y$ in $\G$ then $X$ is an \textit{ancestor} of $Y$, and $Y$ a \textit{descendant} of $X$. We use $pa_{\G}(X)$ to denote the set of parents of $X$ in graph $\G$. 
Idem $ch_{\G}(X)$, $an_{\G}(X)$ and $de_{\G}(X)$ for the sets of children, ancestors, and descendants of $X$ in $\G$, with natural extensions to sets, e.g.\ $pa_{\G}(\bfX): \{V: \exists X \in \bfX, V \in pa_{\G}(X) \}$.
A node $Z$ is a \textit{collider} on a path $\seq{..,X,Z,Y,..}$ if the subpath is of the form $X \tea Z \aet Y$, otherwise it is a \textit{noncollider}. A triple of nodes $\seq{..,X,Z,Y,..}$ on a path is said to be \textit{unshielded} if $X$ and $Y$ are not adjacent in $\G$. An unshielded collider $X \tea Z \aet Y$ is known as a \textit{v-structure}.

A \textit{DG model} is an ordered pair $\seq{\G,\cP}$ where $\G$ is a (cyclic or acyclic) directed graph and $\cP$ is a probability distribution over the vertices (variables) in $\G$. The \textit{global directed Markov property} links the structure of the graph $\G$ to probabilistic independences in $\cP$ via the \textit{d}-separation criterion: for sets of vertices $\bfX,\bfY,\bfZ$ in a graph $\G$, $\bfX$ is \textit{d-connected} to $\bfY$ given $\bfZ$ iff there is an $X \in \bfX$ and $Y \in \bfY$ such that there is a path $\pi$ between $X$ and $Y$ on which every noncollider is not in $\bfZ$, and every collider on $\pi$ is an ancestor of $\bfZ$; otherwise $\bfX$ and $\bfY$ are said to be \textit{d-separated} given $\bfZ$. Two graphs $\G_1$ and $\G_2$ are said to be \textit{d-separation (Markov) equivalent} iff every \textit{d}-separation in $\G_1$ also holds in $\G_2$ and v.v.
For more details on graphical causal models, see \citep{KollerF2009, SGS2000, Pearl2009, Bongers++_AOS_21}. In section 3 in the supplement, we provide more details on Markov properties in structural causal models, and describe some concrete classes of models for which the \textit{d}-separation criterion applies.

\subsection{Features of Cyclic Graphs}
Next we will state a few properties and definitions that are specific to directed graphs with cycles.

\begin{dfn} \label{dfn:scc}
In a directed graph $\G$ over set of vertices $\bfV$, a subset $\bfS \subseteq \bfV$ is a \textbf{strongly connected component (SCC)} of $\G$ iff $\bfS$ is a maximal set of vertices where every vertex is reachable via a directed path in $\G$ from every other vertex in $\bfS$.
\end{dfn}

In cyclic graphs the presence of arcs into directed cycles can create dependencies that behave like additional induced edges:
\begin{dfn} \label{dfn:virtual-edge}
In a graph $\G$, two nodes $A$ and $B$ are said to be \textbf{virtually adjacent} iff there is no edge between $A$ and $B$ in $\G$, but $A$ and $B$ have a common child $C$ which is an ancestor of $A$ or $B$.
\end{dfn}
Two nodes connected by a virtual edge cannot be \textit{d}-separated by any set of nodes, and therefore appear like they are connected by an edge. In \citep{Richardson1997} virtual edges were also called \textit{p(seudo)-adjacent}. 

These induced virtual edges can also be part of paths we have to consider, giving rise to the generalized concept of an itinerary:
\begin{dfn} \label{dfn:itinerary}
In a graph $\G$, a sequence of vertices $\seq{X_0,...,X_{n+1}}$ where all neighbouring nodes in the sequence are (virtually) adjacent in the graph is said to be an \textbf{itinerary}.
If none of the nodes on the itinerary are (virtually)  adjacent to each other except for the ones that occur consecutively on it then the itinerary is said to be \textbf{uncovered}, otherwise it is said to be \textbf{covered}.
\end{dfn}

Virtual edges can also appear in regular (non)collider triples, leading to the generalized notion of (non)conductors:
\begin{dfn} \label{dfn:non+conductor}
In a graph $\G$, a triple $\seq{A,B,C}$ forms a \textbf{conductor} if $\seq{A,B,C}$ is an itinerary, and $B$ is an ancestor of $A$ and/or $C$. If $\seq{A,B,C}$ is an itinerary, but $B$ is NOT an ancestor of $A$ or $C$, then $\seq{A,B,C}$ is a \textbf{nonconductor}.
A (non)conductor $\seq{A,B,C}$ is said to be \textbf{unshielded} if $A$ and $C$ are not (virtually) adjacent, otherwise it is \textbf{shielded}.
\end{dfn}

In some case we can actually detect the presence of some induced edge, although we can never be sure which one:
\begin{dfn} \label{dfn:im+perfect noncond}
In a graph $\G$ a nonconductor triple $\seq{A,B,C}$ is a \textbf{perfect nonconductor} if $B$ is also a descendant of a common child of $A$ and $C$. If not, then $\seq{A,B,C}$ is an \textbf{imperfect nonconductor}.
\end{dfn}
Key notion here is that for unshielded perfect nonconductors conditioning on a set that includes $B$ \textit{always} creates a dependence between $A$ and $C$, whereas unshielded imperfect nonconductors do create a dependence when conditioning on $B$, but not for \textit{every} set containing $B$. This is impossible in acyclic graphs and is therefore a hallmark for the presence of cycles. See the two virtual \textit{v}-structures in Figure \ref{fig1TwoCycle} for an example.

Finally, as pièce de résistance, we have some patterns that introduce a nonlocality aspect:
\begin{dfn} \label{dfn:me-cond}
If $\seq{X_0,...,X_{n+1}}$ is a sequence of vertices such that each consecutive triple along the (uncovered) itinerary is a conductor, and all nodes $\{X_1,..,X_n\}$ are ancestors of each other, but not ancestors of either $X_0$ or $X_{n+1}$, then the triples $\seq{X_0,X_1,X_2}$ and $\seq{X_{n-1},X_n,X_{n+1}}$ are \textbf{mutually exclusive (m.e.) conductors w.r.t. an (uncovered) itinerary}.
\end{dfn}
An example is depicted in Figure \ref{fig2Ustruct}. As a result, graphs that have identical \textit{d}-separation relations locally everywhere in the graph can still differ regarding a \textit{d}-separation between nodes that are arbitrarily far apart in the graph (something that is impossible in the acyclic case).

\subsection{The Cyclic Equivalence Theorem} \label{sub:orgCET}
With the features introduced in the previous section \cite{Richardson1997} established the following characterization:

\textbf{Cyclic Equivalence Theorem (CET)}: Two directed graphs $\G_1$ and $\G_2$ over vertices $\bfV$ are Markov (\textit{d}-separation) equivalent iff
\begin{enumerate}
\item[(i)] they have the same (virtual) adjacencies,
\item[(ii).a] they have the same unshielded conductors,
\item[(ii).b] they have the same unshielded perfect nonconductors,
\item[(iii)] two triples $\seq{A,B,C}$ and $\seq{X,Y,Z}$ are mutually exclusive conductors on some uncovered itinerary $P = \seq{A,B,C,..,X,Y,Z}$ in $\G_1$ iff they are also m.e. conductors on some uncovered itinerary in $\G_2$,
\item[(iv)] if $\seq{A,X,B}$ and $\seq{A,Y,B}$ are unshielded imperfect nonconductors in $\G_1$ and $\G_2$, then $X$ is an ancestor of $Y$ in $\G_1$ iff $X$ is an ancestor of $Y$ in $\G_2$,
\item[(v)] if $\seq{A,B,C}$ and $\seq{X,Y,Z}$ are m.e.\ conductors on an uncovered itinerary $P = \seq{A,B,C,..,X,Y,Z}$, and $\seq{A,M,Z}$ is an unshielded imperfect nonconductor (in $\G_1$ and $\G_2$), then $M$ is a descendant of $B$ in $\G_1$ iff $M$ is a descendant of $B$ in $\G_2$.
\end{enumerate}

\begin{figure*}[t]
\begin{center}
\includegraphics[scale=0.5]{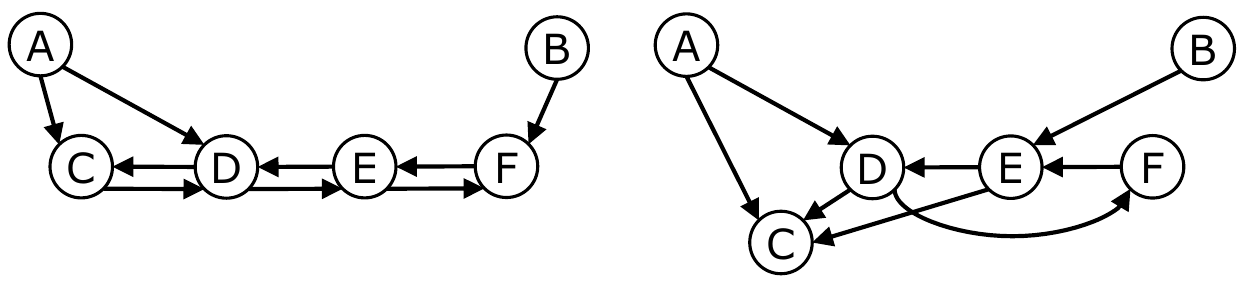}
\includegraphics[scale=0.5]{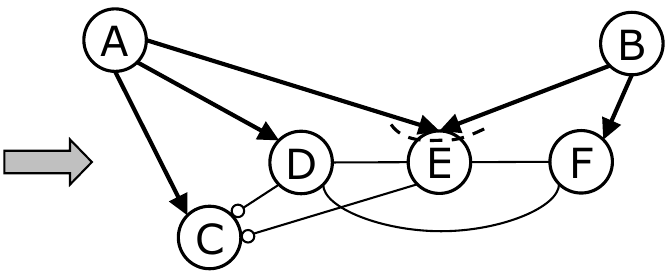}
\caption{\small{Two Markov equivalent graphs (left) with $\seq{A,D,F}$ and $\seq{D,F,B}$ a pair of m.e.\ conductors on uncovered itinerary $\seq{A,D,F,B}$; (right) corresponding (maximally informative) CPAG.}} \label{fig2Ustruct}
\end{center}
\end{figure*}

\subsection{Cyclic PAGs}
To characterize the (\textit{d}-separation) Markov equivalence class of a cyclic directed graph $\G$, denoted $MEC(\G)$, \cite{Richardson1996b_MEC} described an algorithm that created a set of exhaustive lists of instances in the graph matching one of the individual rules in the CET, above. Establishing Markov equivalence then boils down to comparing the lists constructed for each.

Later on, \cite{Richardson1996a_CCD} introduced a more intuitive graphical representation in the form of a (cyclic) partial ancestral graph that also captured enough elements to uniquely identify the equivalence class of a directed graph: 

\begin{dfn} \label{dfn:CPAG}
A graph $\cP$ is a \textbf{partial ancestral graph (PAG)} for directed (a)cyclic graph $\G$ with vertex set $\bfV$, iff
\begin{enumerate}
\item[(i)] there is an edge between vertices $A$ and $B$ iff $A$ and $B$ are d-connected given any subset $\bfW \subseteq \bfV \setminus \{A,B\}$,
\item[(ii)] If $A \tem B$ is in $\cP$, then in every graph in $MEC(\G)$, $A$ is ancestor of $B$,
\item[(iii)] If $A \mea B$ is in $\cP$, then in every graph in $MEC(\G)$, $B$ is NOT an ancestor of $A$,
\item[(iv)] if \mbox{$A \, \ast \!\!\! \relbar \!\!\!$ \underline{$\ast \, B \, \ast$} $ \!\!\! \relbar \!\!\! \ast \, C$} in $\cP$, then $B$ is ancestor of $A$ and/or $C$ in every $\G' \in MEC(\G)$,
\item[(v)] if \mbox{$A \! \relbar \!$ \dashuline{$\!\!\rightarrow \! B \! \leftarrow \!\!$} $ \! \relbar \! C$} in $\cP$, then $B$ is NOT a descendant of a common child of $A$ and $C$ in every $\G' \in MEC(\G)$,
\item[(vi)] any remaining edge mark not oriented in the above ways obtains a circle mark $\cem$ in $\cP$.
\end{enumerate}
We use the term \textbf{cyclic PAG (CPAG)} of a graph $\G$ to denote a PAG $\cP$ that captures invariant ancestral relations shared by all and only the graphs in the Markov equivalence class of $\G$.
\end{dfn}
In these rules the asterisk \mbox{$\ast \!\! \relbar$} mark on an edge is used as a meta symbol that represents any of the other marks $\{-,>,\circ\}$.
The solid underlining in rule (iv), indicating that the middle node is \textit{not} a collider between the other two, is superfluous and therefore often omitted from the graph $\cP$. The dashed underlining in rule (v), however, \textit{is} essential, and unique to cyclic graphs, and appears in the virtual \textit{v}-structures introduced in $\S$\ref{sub:introCMAG}. See Figure \ref{fig2Ustruct} for an example CPAG.

The CPAG has the same purpose and interpretation as the familiar PAG output by the well-known FCI algorithm (\cite{SGS2000, Zhang2008}), including circle marks $X \cem Y$ from rule (vi) to explicitly denote `not determined'. This can be either because the implied ancestral relation is not invariant between all members in the Markov equivalence class of $\G$, i.e.\ there are some graphs where $X$ is an ancestor of $Y$ and some where it is not (`can't know'), or because the relation \textit{is} invariant but we have not determined what it is yet (`don't know'). 
As a result, a graph $\G$ can correspond to different CPAGs $\cP$ that differ in level of completeness. 
In this paper we are not concerned with obtaining the (unique) \textit{maximally informative} CPAG, but instead settle for any \textit{Markov complete} PAG that represents a unique (\textit{d}-separation) Markov equivalence class.

\subsection{CPAG-from-Graph Algorithm} \label{sub:2.5-CPAGfromG}

Using the CPAG definition above we now describe an algorithm by \cite{Richardson1996c_DCCS} that takes as input a (possibly cyclic) directed graph $\G$ and outputs a CPAG $\cP$ such that two graphs $\G_1$ and $\G_2$ are Markov equivalent iff the algorithm outputs the same CPAG for both. In other words, the algorithm is \textit{d-separation complete}.

\begin{enumerate}
\item[(a)] form the complete undirected graph $\cP$ with all circle edges $\cec$, and then for every edge $A \cec B$ in $\cP$, if $A$ is \textit{d}-separated from $B$ given $\bfC = An(\{A,B\}) \setminus \{A,B\}$ then remove edge $A \cec B$ from $\cP$ and record $\bfC$ in $Sepset(A,B)$ and $Sepset(B,A)$,
\item[(b)] for each unshielded triple $A \mem B \mem C$ in $\cP$, orient $A \tea B \aet C$ if $B \notin Sepset(A,C)$,
\item[(c)] for each triple $\seq{A,X,Y}$ such that $X \mem Y$ in $\cP$, $A$ is not adjacent to $X$ or $Y$ in $\cP$, $X \notin Sepset(A,Y)$, orient $X \aet Y$ if $A$ and $X$ are d-connected given $Sepset(A,Y)$,
\item[(d)] for each unshielded triple $A \tea B \aet C$ in $\cP$, if $A$ and $C$ are \textit{d}-separated given a specific set $\bfR$,\footnote{We omit the definition of the set $\bfR$ here for brevity.} then orient \mbox{$A \! \relbar \!$ \dashuline{$\!\!\rightarrow \! B \! \leftarrow \!\!$} $ \! \relbar \! C$} in $\cP$ and record $\bfR$ in $SupSepset\seq{A,B,C}$ (and $SupSepset\seq{C,B,A}$),
\item[(e)] for each quadruple $\seq{A,B,C,D}$, if \mbox{$A \! \relbar \!$ \dashuline{$\!\!\rightarrow \! B \! \leftarrow \!\!$} $ \! \relbar \! C$} in $\cP$, \mbox{$A \tea D \aet C$} or \mbox{$A \! \relbar \!$ \dashuline{$\!\!\rightarrow \! D \! \leftarrow \!\!$} $ \! \relbar \! C$} in $\cP$, $B$ and $D$ are adjacent in $\cP$, then if $D \in SupSepset\seq{A,B,C}$ then orient \mbox{$B \met D$}, otherwise orient $B \tea D$ in $\cP$,
\item[(f)] for each quadruple $\seq{A,B,C,D}$, such that \mbox{$A \! \relbar \!$ \dashuline{$\!\!\rightarrow \! B \! \leftarrow \!\!$} $ \! \relbar \! C$} in $\cP$, and $D$ is not adjacent to both $A$ and $C$ in $\cP$, if $A$ and $C$ are \textit{d}-connected given $SupSepset\seq{A,B,C} \cup \{D\}$, then orient $B \mec D$ as $B \tea D$.
\end{enumerate}

The algorithm has complexity $O(N^7)$, and is \textit{d}-separation complete:

\textbf{Theorem 2} in \citep{Richardson1996c_DCCS}: For two graphs $\G_1$ and $\G_2$, the CPAG-from-Graph algorithm outputs corresponding CPAGs $\cP_1$ and $\cP_2$ that are identical iff $\G_1$ and $\G_2$ are \textit{d}-separation equivalent. 

Actually, the theorem was formulated for the CCD algorithm \citep{Richardson1996a_CCD} for obtaining a CPAG from (oracle) independence information, but the two are so similar that the proof automatically carries over to the CPAG-from-Graph algorithm. 
The algorithm is an improvement by a factor $O(N^2)$ on the earlier list-based Cyclic Classification algorithm in \citep[$\S5.4$]{Richardson1996b_MEC}. 

\section{An ancestral perspective on the CET} \label{sec:3-NewCET}
On reflection of the characterization of Markov equivalence between cyclic graphs obtained, one may note that the rather daunting definitions and terminology in the CET seem to contrast quite sharply with the apparent simplicity of the actual invariant features contained in the CPAG. At the same time complicated again by the fact that some of these `invariant features' like edges in the CPAG are not actually invariant in the underlying graph at all. 

Furthermore, there is no clear match from some rules in the CET to specific invariant features in the CPAG. In particular the `mutually exclusive conductors on an uncovered itinerary'\footnote{Actually this term is a bit of a misnomer, as the two conductors need not be mutually exclusive when there is an induced virtual edge along the uncovered itinerary connecting the two.} in rule CET-(iii) are never explicitly recorded, even though they can of course be inferred from the CPAG afterwards.

A natural question, inspired by the familiar DAG-MAG-PAG triad for acyclic graphs, would be whether it might make sense to also consider an intermediate ancestral stage for cyclic graphs.

In this section we answer that question with an emphatic: yes!  We first introduce the CMAG as the cyclic analogue to the (acyclic) maximal ancestral graph \citep{RichardsonS2002}, and rephrase the CET in terms of ancestral graphs. This results in a simplified set of rules that each are in direct correspondence with invariant features in the CPAG. In the next section we will show that this approach also leads to an efficient procedure to establish Markov equivalence that no longer needs to rely on \textit{d}-separation tests.

\subsection{Introducing the CMAG} \label{sub:introCMAG}

In keeping with the spirit of regular (acyclic) maximal ancestral graphs, we will define a cyclic MAG as: 

\begin{dfn} \label{dfn:CMAG}
The \textbf{cyclic maximal ancestral graph (CMAG)} $\M$ corresponding to (cyclic) directed graph $\G$ over set of vertices $\bfV$ is a graph where:
\begin{enumerate} [label=(\roman*)]
\item there is an edge between every distinct pair of vertices $\{X,Y\}$ iff they cannot be \textit{d}-separated by any subset of $\bfV \setminus \{X,Y\}$ in $\G$, 
\item there is a tail mark $X \tem Y$ at vertex $X$ on the edge to $Y$ iff there exists a directed path from $X$ to $Y$ in $\G$, otherwise there is an arrowhead mark $X \aem Y$, 
\item every unshielded collider triple $X \tea Z \aet Y$ in $\M$ where $Z$ is not a descendant of a common child of $X$ and $Y$ in $\G$ obtains a dashed underline \mbox{$X \! \relbar \!$ \dashuline{$\!\!\rightarrow \! Z \! \leftarrow \!\!$} $ \! \relbar \! Y$}.
\end{enumerate}
Unshielded collider triples without underlining are called \textbf{v-structures}. The `dashed-underlined' collider triples in a CMAG are referred to as \textbf{virtual v-structures}. 
\end{dfn}

With this definition, a CPAG becomes a straightforward collection of invariant edges and edge marks (rather than `ancestral relations') shared by all and only the CMAGs corresponding to graphs in the same Markov equivalence class. 

The `virtual' in the dashed-underlined \textit{v}-structures from rule (iii)  emphasises that they resemble regular \textit{v}-structures in the CMAG, but look and behave differently in the underlying directed graph $\G$. They are a direct consequence of rule (v) in Def.\ \ref{dfn:CPAG}, and correspond to unshielded imperfect nonconductors in $\G$, that are unique to cyclic graphs. In a CMAG $\M$, node $A$ is an \textit{ancestor} of node $B$ (and $B$ a \textit{descendant} of $A$) iff there exists an ancestral path $A \tem .. \tem B$ in $\M$. 

An SCC in directed graph $\G$ corresponds to a \textit{maximal set of nodes in a connected, undirected subgraph} in $\M$, as each node in an SCC is ancestor of all other nodes in the same SCC. Given this one-to-one correspondence we will also use $SCC(Z)$ in the context of a CMAG $\M$ to denote the nodes in the strongly connected component of $Z$ in $\G$.

\subsection{Virtual collider triples}
Having brought out the CMAG we can make a straightforward mapping from elements in the CET to their ancestral counterpart: (virtual) adjacencies become edges, itineraries become paths, unshielded conductors become unshielded noncolliders, unshielded (perfect) nonconductors become v-structures, and unshielded imperfect nonconductors become virtual v-strucutures.

That only leaves the `mutually exclusive conductors w.r.t.\ an uncovered itinerary'. For that we note that these only appear in the CPAG as the invariant arcs into a cycle, oriented in step (c) of the CPAG-from-Graph algorithm.
In other words, from an ancestral perspective it is not about the conductor triples at the beginning and end of the uncovered itinerary, but only about the first and last edge along the corresponding path in the CMAG.

This brings us to the following definition:

\begin{dfn} \label{dfn:u-struct}
In a CMAG $\M$, a quadruple of distinct nodes $\seq{X,Z,Z',Y}$ is a \textbf{u-structure} if there is an uncovered path $X \tea Z \tet .. \tet Z' \aet Y$ in $\M$, where all intermediate nodes are also in $SCC(Z)$.
\end{dfn}
The term \textit{u}-structure reflects the fact that it is similar to a \textit{v}-structure, but with the central collider node replaced by an uncovered path through a strongly connected component. 

There is a straightforward connection between \textit{u}-structures and the `m.e.\ conductors w.r.t.\ an uncovered itinerary' from Definition \ref{dfn:me-cond}:
\begin{lem} \label{lem:me=ustruct}
  For a directed graph $\G$ and corresponding CMAG $\M$, there is a u-structure $\seq{X,Z,Z',Y}$ in $\M$ iff there is an uncovered itinerary $\pi = \seq{X,Z,U,..,U',Z',Y}$ in $\G$, possibly with $Z=U'$ or $U = U'$, where $\seq{X,Z,U}$ and $\seq{U',Z',Y}$ are a pair of m.e.\ conductors w.r.t.\ the uncovered itinerary $\pi$ in $\G$.
\end{lem}
(For proof details for this and other results in the rest of this article, see supplement.)

Crucially, in the CMAG or CPAG we do not actually record the \textit{u}-structure explicitly. In fact, the only elements of a \textit{u}-structure that need to be oriented in the CPAG are the first and last edge \textit{into} the strongly connected component (cf. step (c) of the CPAG-from-Graph algorithm, \S\ref{sub:2.5-CPAGfromG}).

As a result, we do not  have to identify the full quadruple $\seq{X,Z,Z',Y}$ of each \textit{u}-structure, but only if an edge $X - Z$ is part of \textit{some} u-structure pattern. For that, we can rely on the following result:

\begin{lem} \label{lem:u-struct-undir}
In a CMAG $\M$, a pair of nodes $\seq{X,Z}$ is part of a u-structure $\seq{X,Z,Z',Y}$ with a node $Y \in \bfY \subseteq pa(SCC(Z)) \setminus adj(\{X,Z\})$, iff $X \in pa(Z)$, and $X$ and $Y$ are connected in the subgraph over $((SCC(Z) \setminus adj(X)) \cup \{X,Z\} \cup \bfY$.
\end{lem}

This significantly reduces the complexity of establishing Markov equivalence later on, as it means we only need to search over triples rather than quadruples in the CMAG. 
More importantly, it motivates the introduction of the following invariant element, which in turn will significantly simplify the CET.

\begin{dfn} \label{dfn:virtual-col-triple}
In a CMAG $\M$, a triple of distinct nodes $\seq{X,Z,Y}$ is a \textbf{virtual collider triple} iff $\seq{X,Z,Y}$ is a virtual \textit{v}-structure, or there is some $Z' \in SCC(Z)$, such that either $\seq{X,Z,Z',Y}$ or $\seq{X,Z',Z,Y}$ is a \textit{u}-structure.
\end{dfn}
Intuitively, a virtual collider triple $\seq{X,Z,Y}$ implies that $X$ and $Y$ are connected by an uncovered itinerary via nodes in $SCC(Z)$ that \textit{identifiably} contains one or more virtual edges. The strongly connected component of $Z$ fulfils the role of collider in $X \tea SCC(Z) \aet Y$, and the \textit{virtual} emphasises  there is no `real' collider triple $X \tea Z \aet Y$ in the underlying directed graph.

\begin{figure}[h]
  \centering
  \includegraphics[width=0.9\linewidth,page=1]{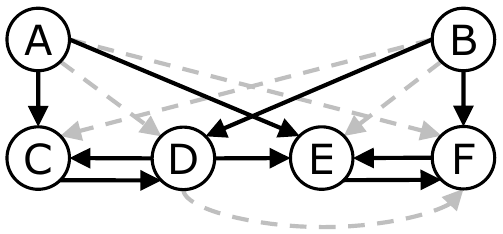}
  \caption{\small Example CET orientation rule (iv) on virtual collider triples $\seq{A,D,B}$ and $\seq{A,E,B}$ for invariant edge $D \tea E$, with virtual edges as dashed grey arcs.} 
  \label{fig3CET-4}
\end{figure}

\subsection{A new CET}
We are now ready to restate the Cyclic Equivalence Theorem in terms of CMAGs:

\begin{thm} \label{thm:CET-CMAG}
Two directed graphs $\G_1$ and $\G_2$, corresponding to CMAGs $\M_1$ and $\M_2$, are Markov (\textit{d}-separation) equivalent iff
\begin{enumerate}  
\item[(i)] $\M_1$ and $\M_2$ have the same skeleton,
\item[(ii)] $\M_1$ and $\M_2$ have the same \textit{v}-structures,
\item[(iii)] $\M_1$ and $\M_2$ have the same virtual collider triples,
\item[(iv)] if $\seq{A,B,C}$ is a virtual collider triple, and $\seq{A,D,C}$ a virtual \textit{v}-structure, then $B$ is an ancestor of $D$ in $\M_1$ iff $B$ is an ancestor of $D$ in $\M_2$.\footnote{In the original published version, $\seq{A,D,C}$ was erroneously included as a virtual collider triple, but the distinction is needed to restrict the pairs of virtual collider triples to check.}
\end{enumerate}
In this case, we call $\M_1$ and $\M_2$ `Markov equivalent'.
\end{thm}

Each rule in this ancestral CET can be linked directly to specific invariant elements in the CPAG: rule (i) to the edges in the CPAG, (ii) to all \textit{v}-structures, (iii) to remaining invariant arcs into strongly connected components (incl. under-dashed marks for virtual \textit{v}-structures), and (iv) to invariant edges within or between (identifiable) cycles. 
 
Comparing to the original CET in $\S$\ref{sub:orgCET}, we can see that the ancestral formulation greatly simplifies the Markov equivalence characterization, leading to two fewer rules and only requiring (collider) triples.

An interesting observation is that in the acyclic case going from DAGs to MAGs (to allow for unobserved confounders) implied going from `(unshielded) collider triples' (\textit{v}-structures) to `collider triples with order' in the characterization of Markov equivalence between graphs \citep{Ali++2009,ClaassenB2022}. Given that analogy we conjecture that for the cyclic case allowing for latent confounders can similarly be accomplished by extending to `(virtual) collider triples with order'. 

\section{Establishing Markov Equivalence for cyclic graphs} 
\label{sec:4-MarkovEq}
We now show that with the intermediate CMAG representation we can derive a consistent CPAG that uniquely defines the equivalence class of a (cyclic) directed graph without the need for any \textit{d}-separation tests. 
The resulting algorithm is extremely fast, and allows to determine Markov equivalence between graphs by directly comparing the output CPAGs.

\subsection{Obtaining the CMAG}

To capture the first rule of the new CET, we need to obtain the skeleton of the CPAG. To avoid the \textit{d}-separation tests in step (a) of the CPAG-from-Graph algorithm in $\S$\ref{sub:2.5-CPAGfromG}, we can use the following result:

\begin{lem} \label{lem:skeleton}
In a CMAG $\M$ corresponding to directed graph $\G$, two variables $X$ and $Y$ are adjacent, iff $X$ and $Y$ are (virtually) adjacent in $\G$.
\end{lem}

It implies we can read off the CMAG skeleton directly from the graph $\G$, by starting from the skeleton of $\G$, and adding an edge between $X$ and $Y$ for every \textit{v}-structure $X \tea Z \aet Y$ in $\G$ with $Y \in SCC_{\G}(Z)$.

It does mean that we first need to partition the vertices in the graph into the set of strongly connected components. This can be achieved in time linear in the number of vertices and edges $O(N d)$ using e.g. Tarjan's algorithm \citep{Tarjan1972}.\footnote{Actually, we use a modified version of Tarjan's algorithm that also tracks ancestral relations in one go. For details on this and all other algorithms used in the paper, see source code available at \url{https://github.com/tomc-ghub/CET_uai2023}}

Subsequent orientations of edges in $\M$ follow orientations in $\G$, where edges between nodes in the same $SCC$ become undirected edges, signifying they are all ancestor of each other.
Induced edges between nodes in the same cycle also become undirected, and induced edges by a triple $X \tea Z \aet Y$ in $\G$ with $X \notin SCC_{\G}(Z)$ become $X \tea Y$.

Alternatively, we can process each node $X$ in $\G$ in turn, and draw undirected edges between all of its parents in the same cycle as $X$ (incl.\ $X$)  in $\M$, and add arcs from all remaining parents into the first set of parents (again incl.\ $X$), which is what we do in Algorithm 1, below.

\begin{algorithm}[h]
  \caption{Graph-to-CMAG}   \label{alg:CG-to-CMAG}
\begin{algorithmic}		
   \STATE{\bfseries Input:} directed cyclic graph $\G$ over nodes $\bfV$
   \STATE{\bfseries Output:} CMAG $\M$, $SCC$s, 
   \STATE $SCC \gets Get\_StronglyConnComps(\G)$
   \STATE  \textit{part 1: CMAG rules (i) + (ii)}
   \FORALL{$X \in \bfV$}
   \STATE $\bfZ \gets pa_{\G}(X)$
   \STATE $\bfZ_{cyc} \gets \bfZ \cap SCC(X)$
   \STATE $\bfZ_{acy} \gets \bfZ \setminus \bfZ_{cyc}$
   \STATE add all arcs $\bfZ_{acy} \tea \bfZ_{cyc} \cup \{X\}$ to $\M$
   \STATE add all undirected edges $\bfZ_{cyc} \tet \bfZ_{cyc} \cup \{X\}$ to $\M$
   \ENDFOR
   \STATE  \textit{part 2: CMAG rule (iii)}
   \FORALL{$X \in \bfV: |SCC(X)|\geq 2$}
   \STATE $\bfZ \gets pa_{\M}(X)$
   \FORALL{non-adjacent pairs $\{Z_i,Z_j\} \subseteq \bfZ$}
	\STATE \textbf{if}   $\{Z_i,Z_j\} \nsubseteq adj_{\G}(X)$ \textbf{then}
	\STATE \textbf{~~~if}   $X \notin de_{\G}(ch_{\G}(Z_i) \cap ch_{\G}(Z_j))$  \textbf{then}
    \STATE ~~~~~~mark virtual v-structure $\seq{Z_i,X,Z_j}$ in $\M$
   \ENDFOR
   \ENDFOR
\end{algorithmic}
\end{algorithm}

The second part of Algorithm \ref{alg:CG-to-CMAG} simply involves checking all \textit{v}-structures in $\M$ with central collider node in a non-trivial SCC, and with at least one virtual edge in $\G$. Here we use the matrix of ancestral relations, constructed when identifying the SCCs at the start of the algorithm, to reduce the `descendant of' check in the second `if'-clause to constant time per node.

\subsection{Constructing the CPAG}

Before we can go on to construct a CPAG from the CMAG $\M$ obtained above, we still need to recognise the virtual collider triples corresponding to so-called \textit{u}-structures. These are not marked explicitly in the CMAG (contrary to virtual \textit{v}-structures), but they \textit{are} needed to orient certain invariant edges in the CPAG corresponding to rules (iii) and (iv) in Theorem \ref{thm:CET-CMAG}. Fortunately, for that we can rely on Lemma \ref{lem:u-struct-undir}, where the fact that we only need to consider straightforward `connected subgraphs' means the complexity of this step scales linearly with the number of edges in the subgraph.

It also means that, in the construction of the CPAG, to cover invariant arcs from \textit{u}-structures, we only need to consider edges $X \tea Z$ in $\M$ that are not yet oriented in $\cP$, and where $Z$ is part of a nontrivial SCC (size $|SCC(Z)| \geq 2$), and the $\bfY$ in Lemma \ref{lem:u-struct-undir} are all other parents of $SCC(Z)$ that are not adjacent to $X$ and/or $Z$ in $\M$.
 Note that the arcs oriented thusly were previously captured by the exhaustive search in step (c) of the CPAG-from-Graph algorithm in section \ref{sub:2.5-CPAGfromG}.  

We can now bring these steps together in Algorithm \ref{alg:G-to-CPAG}.\footnote{The second clause in the `if' statement in part 2 was added as a result of the correction to CET rule (iv).}

\begin{algorithm}[h]
  \caption{Graph-to-CPAG}   \label{alg:G-to-CPAG}
\begin{algorithmic}		
   \STATE{\bfseries Input:} directed cyclic graph $\G$ over nodes $\bfV$, 
   \STATE{\bfseries Output:} CPAG $\cP$, 
   \STATE $(\M,SCC) \gets$ \mbox{Graph-to-CMAG}$(\G)$ 
   \STATE \textit{part 1: new-CET rules (i)-(iii)}
   \STATE $\cP \gets$ skeleton of $\M$ with all $\cec$ edges
   \STATE $\cP \gets$ copy all (virtual) \textit{v}-structures from $\M$ 
   \FORALL{$X \cec Z$ in $\cP$, $X \tea Z$ in $\M$, $|SCC(Z)| \geq 2$}
   \STATE \textbf{if}   $\exists \seq{X,Z,Z',Y}$ as \textit{u}-structure in $\M$ \textbf{then}
   \STATE ~~~orient $X \tea Z$ in $\cP$ ~~~~~~~~~~~~~~~ \COMMENT{\textit{Lemma \ref{lem:u-struct-undir}}}
   \ENDFOR
   \STATE \textit{part 2: new-CET rule (iv)}
   \FORALL{$Z \cec W$ at virtual \textit{v}-structures $\seq{X,Z,Y}$}
   \STATE \textbf{if}   $\seq{X,W,Y}$ is virtual collider triple \textbf{or}\\
   ~~~ $\exists\seq{X,B,Y}$ as virt.\ coll.\ triple, with $B \notin an(Z)$,\\
        ~~~ uncovered $B \tem .. \tem W \aet Z$ in $\M$, and \\
        ~~~ $\nexists U:\{$\textit{v}-structure $X \tea U \aet Y$ in $\M, W \in de(U)\}$ 
   \STATE \textbf{then} copy edge $Z \mem W$ from $\M$ to $\cP$ 
   \ENDFOR
\end{algorithmic}
\end{algorithm}

In practice we already copy invariant features to the CPAG while constructing the CMAG to improve efficiency. Note that the final output CPAG is \textit{d}-separation complete, but \textit{not} guaranteed to be identical to the CPAG from the original CPAG-from-Graph algorithm. This is because steps (c) and (f) there contain an exhaustive search that also orients certain arcs that are sound but not needed for the CET, but could also be obtained from subsequent implied orientation rules, similar to the PC/FCI algorithm. Therefore the CPAGs from the two algorithms cannot be compared directly against each other to establish Markov equivalence. However the main result remains the same: 

\begin{thm}
For two different directed graphs $\G_1$ and $\G_2$, let $\cP_1$ and $\cP_2$ be the corresponding CPAGs output by algorithm \ref{alg:G-to-CPAG}. Then $\G_1$ is Markov (\textit{d}-separation) equivalent to $\G_2$ iff $\cP_1 = \cP_2$.
\end{thm}

\subsection{Computational complexity} \label{sub:complexity}
The scaling behaviour of Algorithm \ref{alg:G-to-CPAG} depends primarily on the number of vertices $N$ and average node degree $d$ corresponding to $N*d$ edges in the graph. 

The first part of algorithm \ref{alg:CG-to-CMAG} requires order $O(N + N*d)$ steps to find the strongly connected components, followed by a loop over $N$ vertices comparing $d^2$ parents, so overall $O(N*d^2)$.
Similarly, the second part of algorithm \ref{alg:CG-to-CMAG} considers $d^2$ parents for $N$ nodes, checking it is not a descendant of $d$ possible common children for $O(N*d^3)$ (provided the $Get\_StronglyConnComps$ step also tracks the ancestral matrix for constant-time descendant checks).

Next, the first two steps in part 1 of algorithm \ref{alg:G-to-CPAG}, initializing the skeleton and (virtual) v-structures, are also $O(N*d^2)$.
Next, for the \textit{u}-structures we may need to loop over $O(N*d)$ edges and establish connectedness in a subgraph over at most $N$ nodes, which can be done in order $O(N*d)$ steps (similar to the SCC procedure) leading to overall $O(N^2*d^2)$.
Finally, in part 2 of algorithm \ref{alg:G-to-CPAG} we need to loop over $N*d$ edges to virtual \textit{v}-structures, considering $d^2$ other virtual collider triples (previously identified in part 1), each with links to $d$ candidate nodes $B$, followed by a (single) test for connectedness per edge, order $O(N*d)$, giving a total of $O(N*d*(d^3 + N*d)$.

So overall worst case complexity scales with $O(N^5)$ for arbitrary density, which is a significant improvement over the $O(N^7)$ achieved by the current state-of-the-art CPAG-from-Graph algorithm.

In practice, even for large graphs there is typically only a relatively small number of cases to consider in the final steps, and so for both procedures the actual scaling behaviour is usually much better than this worst-case bound suggests, as evidenced by the next section.

\section{Experimental evaluation} \label{sec:5-ExpEval}
In order to evaluate the performance of the CPAG-from-graph procedure as a function of size and density of the graph we generate collections of random directed cyclic graphs and track both average and worst-case performance in terms of number of elementary operations and time.

Note that in generating the random cyclic graphs we introduced a few parameters to be able to tweak the number and type of cycles included, as for increasing size and density truly random cyclic graphs quickly tend to collapse into the `one big cycle' type, avoiding most of the intricacies from CET rules (iv) and (v) that relate to invariate edges between cycles; see section 1.1 in the supplement for details.

\subsection{Scaling behaviour}

\begin{figure}[h]
  \centering
  \includegraphics[width=1.0\linewidth]{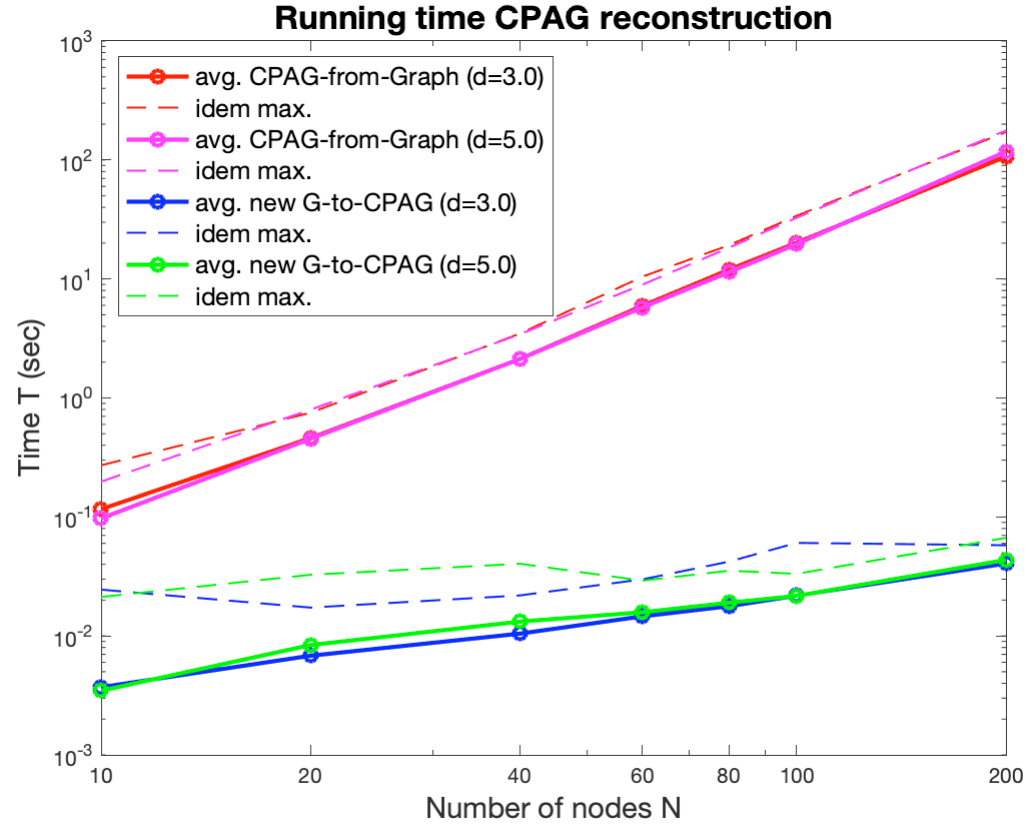}
  \caption{\small Log-log plot depicting scaling behaviour of original (red/magenta) and new CPAG algorithms (blue/green), as a function of size of the graph $N$, for two different densities $d\in \{3.0,5.0\}$. Solid lines indicate average performance over 100 instances, dashed lines the worst case encountered.} 
  \label{fig4PerfResults}
\end{figure}

Figure \ref{fig4PerfResults} shows the results for the two CPAG-from-graph procedures. As expected, the scaling behaviour of the new procedure in Algorithm \ref{alg:G-to-CPAG} is much more benign. 
As a result, for graphs of $N=200$ nodes with density $d=3.0$, the latter requires only about $0.05$ sec. on average to construct the CPAG, whereas the original version takes about $78$ sec.: a speed-up by 3 orders of magnitude.

In the supplement we see that the original CPAG-from-Graph procedure spends the vast majority of its time in the expensive \textit{d}-separation searches in stage (a) and (c), whereas for sparse graphs the new Graph-to-CPAG version spends roughly equal amounts in each phase. For denser graphs, the final stage in the latter starts to dominate, as expected from the complexity analysis in section \ref{sub:complexity}.

Finally, note that for both algorithms in this experiment there is not much difference between average and worst-case scaling behaviour in the collection of randomly sampled graphs (around $1.5-2.0$ times more expensive for both versions), and both stay well below their theoretical worst-case limits.
The reason is that, in order to reach the dreaded `worst case' scenario, the graphs require very specific configurations that are extremely unlikely to occur in truly random graphs. As a result, a reassuring message of Figure \ref{fig4PerfResults} is that in practice the challenge of handling even (very) large cyclic directed graphs is likely to remain feasible in practice, despite the quite imposing theoretical worst-case limit.

\section{Discussion}

We presented a new, ancestral perspective on the Cyclic Equivalence Theorem for directed graphs that resulted in a fast and efficient procedure to obtain the CPAG from an arbitrary directed graph.

The resulting CPAGs can be compared directly to establish Markov equivalence between cyclic directed graphs, but so far we made no attempt to derive \textit{all} invariant features shared by all (and only) the CMAGs in the same equivalence class. In other words, we did not yet aim for the \textit{maximally informative} CPAG. As a result, not all identifiable cycles are guaranteed to appear in an easily recognisable form. Squeezing out all available information would likely entail a set of additional orientation propagation rules, similar to augmented FCI in \citep{Zhang2008}. 

The obtained efficiency of the Graph-to-CPAG procedure in algorithm \ref{alg:G-to-CPAG} also means it is fast enough to be a viable route for extending score-based greedy equivalence search algorithms like GES \citep{Chickering2002} towards cyclic graphs, similar to recent extensions for acyclic graphs in the presence of confounders \citep{ClaassenB2022}. 

However, we consider the most promising aspect of our results the significantly reduced conceptual complexity provided by the ancestral perspective. The new ancestral CET is notably simpler than the original version, and suggests a natural extension to cyclic models with confounders, analogous to that for MAGs.

Finally, the CMAG under \textit{d}-separation treats strongly connected components more similar to the nonlinear case under $\sigma$-separation \citep{MooijC2020}, which suggests they may be merged to handle arbitrary cyclic relationships in the near future. We hope this may encourage researchers to renew work towards extending available constraint-based algorithms towards sound and complete causal discovery in the presence of confounders, cycles, and selection bias.

\bibliography{uai2023-newCET}

\begin{thebibliography}{}

\bibitem[Ali et~al., 2009]{Ali++2009}
Ali, R.~A., Richardson, T.~S., and Spirtes, P. (2009).
\newblock Markov equivalence for ancestral graphs.
\newblock {\em The Annals of Statistics}, 37(5B):2808--2837.

\bibitem[Bongers et~al., 2021]{Bongers++_AOS_21}
Bongers, S., Forr{\'e}, P., Peters, J., and Mooij, J.~M. (2021).
\newblock Foundations of structural causal models with cycles and latent
  variables.
\newblock {\em Annals of Statistics}, 49(5):2885--2915.

\bibitem[Chickering, 2002]{Chickering2002}
Chickering, D.~M. (2002).
\newblock Optimal structure identification with greedy search.
\newblock {\em Journal of machine learning research}, 3(Nov):507--554.

\bibitem[Claassen and Bucur, 2022]{ClaassenB2022}
Claassen, T. and Bucur, I.~G. (2022).
\newblock Greedy equivalence search in the presence of latent confounders.
\newblock In {\em Uncertainty in Artificial Intelligence}, pages 443--452.
  PMLR.

\bibitem[Forr{\'e} and Mooij, 2017]{ForreMooij_1710.08775}
Forr{\'e}, P. and Mooij, J.~M. (2017).
\newblock Markov properties for graphical models with cycles and latent
  variables.
\newblock {\em arXiv.org preprint}, arXiv:1710.08775 [math.ST].

\bibitem[Forr{\'e} and Mooij, 2018]{ForreM2018}
Forr{\'e}, P. and Mooij, J.~M. (2018).
\newblock Constraint-based causal discovery for non-linear structural causal
  models with cycles and latent confounders.
\newblock In {\em Proceedings of the 34th Annual Conference on {U}ncertainty in
  {A}rtificial {I}ntelligence ({UAI}-18)}.

\bibitem[Hu and Evans, 2020]{HuE2020}
Hu, Z. and Evans, R. (2020).
\newblock Faster algorithms for markov equivalence.
\newblock In {\em Conference on Uncertainty in Artificial Intelligence}, pages
  739--748. PMLR.

\bibitem[Hyttinen et~al., 2012]{HyttinenEH2012}
Hyttinen, A., Eberhardt, F., and Hoyer, P. (2012).
\newblock Learning linear cyclic causal models with latent variables.
\newblock {\em Journal of Machine Learning Research}, 13:3387--3439.

\bibitem[Koller and Friedman, 2009]{KollerF2009}
Koller, D. and Friedman, N. (2009).
\newblock {\em Probabilistic graphical models: principles and techniques}.
\newblock MIT press.

\bibitem[Koster, 1996]{Kos96}
Koster, J. (1996).
\newblock {M}arkov properties of nonrecursive causal models.
\newblock {\em The Annals of Statistics}, 24(5):2148--2177.

\bibitem[Lacerda et~al., 2008]{LacerdaSRH2008}
Lacerda, G., Spirtes, P., Ramsey, J., and Hoyer, P.~O. (2008).
\newblock Discovering cyclic causal models by independent components analysis.
\newblock In {\em Proceedings of the 24th Conference on {U}ncertainty in
  {A}rtificial {I}ntelligence ({UAI}-08)}.

\bibitem[Mooij and Claassen, 2020]{MooijC2020}
Mooij, J.~M. and Claassen, T. (2020).
\newblock Constraint-based causal discovery using partial ancestral graphs in
  the presence of cycles.
\newblock In {\em Conference on Uncertainty in Artificial Intelligence}, pages
  1159--1168. PMLR.

\bibitem[Mooij and Heskes, 2013]{MooijHeskes_UAI_13}
Mooij, J.~M. and Heskes, T. (2013).
\newblock Cyclic causal discovery from continuous equilibrium data.
\newblock In Nicholson, A. and Smyth, P., editors, {\em Proceedings of the 29th
  Annual Conference on {U}ncertainty in {A}rtificial {I}ntelligence
  ({UAI}-13)}, pages 431--439. AUAI Press.

\bibitem[Mooij et~al., 2011]{Mooij_et_al_NIPS_11}
Mooij, J.~M., Janzing, D., Heskes, T., and Sch{\"o}lkopf, B. (2011).
\newblock On causal discovery with cyclic additive noise models.
\newblock In Shawe-Taylor, J., Zemel, R., Bartlett, P., Pereira, F., and
  Weinberger, K., editors, {\em {A}dvances in {N}eural {I}nformation
  {P}rocessing {S}ystems 24 ({NIPS}*2011)}, pages 639--647.

\bibitem[Neal, 2000]{Nea00}
Neal, R. (2000).
\newblock On deducing conditional independence from $d$-separation in causal
  graphs with feedback.
\newblock {\em Journal of Artificial Intelligence Research}, 12:87--91.

\bibitem[Pearl, 2009]{Pearl2009}
Pearl, J. (2009).
\newblock {\em Causality: Models, Reasoning and Inference}.
\newblock Cambridge University Press.

\bibitem[Pearl and Dechter, 1996]{PearlD1996}
Pearl, J. and Dechter, R. (1996).
\newblock Identifying independence in causal graphs with feedback.
\newblock In {\em Proceedings of the 12th Annual Conference on {U}ncertainty in
  {A}rtificial {I}ntelligence ({UAI}-96)}, pages 420--426.

\bibitem[Richardson, 1996a]{Richardson1996c_DCCS}
Richardson, T. (1996a).
\newblock Discovering cyclic causal structure.
\newblock Technical Report CMU-PHIL-68, Carnegie Mellon University.

\bibitem[Richardson, 1996b]{Richardson1996a_CCD}
Richardson, T. (1996b).
\newblock A discovery algorithm for directed cyclic graphs.
\newblock In {\em Proceedings of the Twelfth international conference on
  Uncertainty in Artificial Intelligence ({UAI}-96)}, pages 454--461.

\bibitem[Richardson, 1996c]{Richardson1996b_MEC}
Richardson, T. (1996c).
\newblock A polynomial-time algorithm for deciding {M}arkov equivalence of
  directed cyclic graphical models.
\newblock In {\em Proceedings of the Twelfth international conference on
  Uncertainty in artificial intelligence}, pages 462--469.

\bibitem[Richardson, 1997]{Richardson1997}
Richardson, T. (1997).
\newblock A characterization of {M}arkov equivalence for directed cyclic
  graphs.
\newblock {\em International Journal of Approximate Reasoning},
  17(2-3):107--162.

\bibitem[Richardson and Spirtes, 2002]{RichardsonS2002}
Richardson, T.~S. and Spirtes, P. (2002).
\newblock Ancestral graph {M}arkov models.
\newblock {\em The Annals of Statistics}, 30(4):962--1030.

\bibitem[Rothenh\"{a}usler et~al., 2015]{RothenhauslerHPM2015}
Rothenh\"{a}usler, D., Heinze, C., Peters, J., and Meinshausen, N. (2015).
\newblock {BACKSHIFT}: Learning causal cyclic graphs from unknown shift
  interventions.
\newblock In {\em Advances in Neural Information Processing Systems 28 ({NIPS}
  2015)}, pages 1513--1521.

\bibitem[Spirtes, 1994]{Spirtes1994}
Spirtes, P. (1994).
\newblock Conditional independence in directed cyclic graphical models for
  feedback.
\newblock Technical Report CMU-PHIL-54, Carnegie Mellon University.

\bibitem[Spirtes, 1995]{Spirtes1995}
Spirtes, P. (1995).
\newblock Directed cyclic graphical representations of feedback models.
\newblock In {\em Proceedings of the Eleventh Conference on {U}ncertainty in
  {A}rtificial {I}ntelligence ({UAI}-95)}, pages 499--506.

\bibitem[Spirtes et~al., 2000]{SGS2000}
Spirtes, P., Glymour, C., and Scheines, R. (2000).
\newblock {\em Causation, Prediction, and Search}.
\newblock MIT press, 2nd edition.

\bibitem[Strobl, 2018]{Strobl2018}
Strobl, E.~V. (2018).
\newblock A constraint-based algorithm for causal discovery with cycles, latent
  variables and selection bias.
\newblock {\em International Journal of Data Science and Analytics}, 8:33--56.

\bibitem[Tarjan, 1972]{Tarjan1972}
Tarjan, R. (1972).
\newblock Depth-first search and linear graph algorithms.
\newblock {\em SIAM journal on computing}, 1(2):146--160.

\bibitem[Wien{\"o}bst et~al., 2022]{WienobstBL2022}
Wien{\"o}bst, M., Bannach, M., and Li{\'s}kiewicz, M. (2022).
\newblock A new constructive criterion for markov equivalence of mags.
\newblock In {\em Uncertainty in Artificial Intelligence}, pages 2107--2116.
  PMLR.

\bibitem[Wright, 1921]{Wright1921}
Wright, S. (1921).
\newblock Correlation and causation.
\newblock {\em Journal of Agricultural Research}, 20:557--585.

\bibitem[Zhang, 2008]{Zhang2008}
Zhang, J. (2008).
\newblock On the completeness of orientation rules for causal discovery in the
  presence of latent confounders and selection bias.
\newblock {\em Artificial Intelligence}, 172(16-17):1873--1896.

\end{thebibliography}

\newpage


\title{Supplement - Establishing Markov Equivalence in Cyclic Directed Graphs}
\maketitle

\begin{abstract}
This part contains the revised version of the supplement to the original UAI2023 publication `Establishing Markov Equivalence in Cyclic Directed Graphs'. It includes the correction to rule (iv) in Theorem 1 and the subsequent adjustment in Algorithm 2, as well as a number of extensions that should make the proofs essentially self-contained. Numbering and notations follow the main article.
\end{abstract}

\section{Additional experimental results} \label{sec:7-ExpEval}
This section elaborates on the random cyclic graph generating process, and a result that offers some added insight into the inner workings of the two CPAG algorithms.

\subsection{Generating random cyclic graphs}
In contrast to the familiar acyclic graphs, in cyclic graphs there can be \textit{two} edges between each pair of nodes, corresponding to a total of $N(N-1)$ possible directed edges for graphs over $N$ nodes. However, in both the Erdos-Renyi model (all graphs with $n$ edges equally likely) and the Gilbert model (all edges appear with equal probability $p$), as density or size of the graph increases, the resulting graph is overwhelmingly likely to contain just one, big strongly connected component, with only a few other nodes on its periphery. As a key part of the CET is about invariant edges \textit{between} components in rule (iv) (see e.g.\ Figure 3 in the main article), just evaluating on arbitrary random graphs would likely lead to an incomplete or biased perspective. In addition, a number of challenges in finding the correct CPAG are related to sequences of connected two-cycles (see main, Figure 2), which in larger fully random graphs are also exceedingly unlikely to appear.

Therefore we tweak the random graph generating process to allow some control over the number and size of the strongly connected components. We introduce a 3-stage process parameterized by size $N$ and density $d$, as well as parameters $p_{two}$ for the proportion of two-cycles, and $p_{acy}$ and $p_{cyc}$ for the proportion of recursive resp.\ nonrecursive edges that remain:
\begin{enumerate}
\item randomly sample the required number of two-cycles, 
\item add random arcs from lower to higher numbered nodes,
\item add completely random arcs for the remaining edges.
\end{enumerate}
Afterwards a random permutation of the nodes is applied to ensure there is no implicit bias in the ordering.

With this procedure, setting $[p_{two},p_{acy}, p_{cyc}] = [0,1,0]$ would lead to a random acyclic graph, whereas setting $[0.1,0.9,0]$ would lead to a random acyclic graph with some edges turned into two-cycles. Setting $[0,0,1]$ would lead to a fully random cyclic graph in the Erdos-Renyi model.
In practice setting e.g.\ $[p_{two},p_{acy}, p_{cyc}] = [0.1,0.82,0.08]$ leads to a varied number and size of the strongly connected components for graphs of up to $N=200$ nodes with density $d=3.0$. For $N=200$ this leads on average to about 11 nontrivial strongly connected components with average largest component size of about 17 vertices.

For larger/higher density graphs the $p_{cyc}$ proportion should be reduced to avoid collapsing into the `one big cycle' trap. In our experiments for $d=5.0$ we used $[p_{two},p_{acy}, p_{cyc}] = [0.05,0.93,0.02]$, which, for $N=200$ resulted on average in about 5 nontrivial strongly connected components, with an average largest size of about 70 vertices.

Additional implementation details will be published with the accompanying source code.

\subsection{Relative time spent per stage}
To take a closer look at the relative contribution of each stage in the two different CPAG procedures to the overall time complexity we also timed each stage separately. Average results are depicted below.

\begin{figure}[h]
  \centering
  \includegraphics[width=1.0\linewidth]{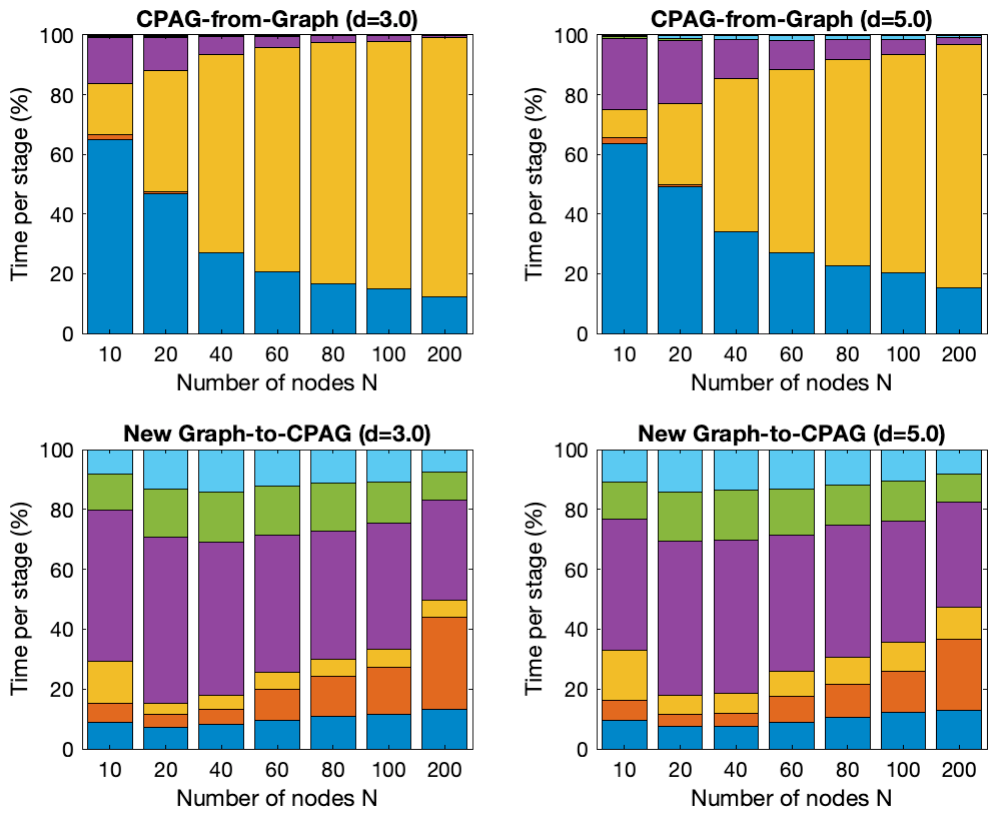}
\renewcommand\thefigure{5}	
  \caption{\small Plots depicting the relative proportion each algorithm spends on average in the different stages, as a function of the size of the graph $N$, for two different densities $d \in \{3.0, 5.0\}$. Stages are ordered bottom up, i.e.\ first stage on the x-axis, second stage on top of that etc.} 
  \label{fig5PerfResults}
\end{figure}

We see that the original CPAG-from-Graph procedure spends the vast majority of its time in the expensive \textit{d}-separation searches in stage (a) (blue) and (c) (yellow), whereas the new Graph-to-CPAG version spends roughly constant amounts in each phase. For denser graphs, the final stage (cyan) in the latter is theoretically the most expensive worst case, but remains at nearly constant proportion in practice, as it is extremely rare to encounter such instances in arbitrary random graphs.

The correction to rule (iv) in Theorem 1 resulted in an additional clause in the `if' statement in part 2 of the Graph-to-CPAG algorithm, but also allowed for more efficient filtering of candidate edges to check. As a result, the implementation now spends a bit more time preparing in the 4th stage (purple), but significantly less in the final stage (cyan), making the new implementation as a whole scale slightly better than the original version (about twice as fast for $N=200$).

\section{Proof details}\label{sec:8-Proofs}
\newtheorem{lemma}{Lemma}
\setcounter{lemma}{3}  
First a few results on properties of \textit{ancestral paths} in CMAGs, i.e.\ paths of the form $X_1 \tem .. \tem X_n$, so that every vertex $X_i$ is ancestor of all $X_{>i}$, and in particular$X_1$ is ancestor of $X_n$, and $X_n$ is a descendant of $X_1$.

\textbf{Lemma 3} \textit{In a CMAG $\M$ corresponding to directed graph $\G$, two variables $X$ and $Y$ are adjacent, iff $X$ and $Y$ are (virtually) adjacent in $\G$.}
\begin{proof}
This is Lemma 1 in \citep{Richardson1997}.
\end{proof}

\begin{lemma} \label{lemAncCMAG}
In a CMAG $\M$ corresponding to directed graph $\G$,  $A$ is ancestor of $B$ in $\M$ iff $A$ is ancestor of $B$ in $\G$.
\end{lemma}
\begin{proof}
If $A$ is ancestor of $B$ in $\M$, then there exists a path $\pi = \seq{A = X_0 \tem (X_1 \tem .. \tem X_{n-1}) \tem X_n = B}$ in $\M$. By Definition 8-(ii), each edge $X_i \tem X_{i+1}$ along $\pi$ in $\M$ implies a directed  path in $\G$ from $X_i$ to $X_{i+1}$. Concatenating them provides a directed path from $A$ to $B$ in $\G$ which implies $A$ is ancestor of $B$ in $\G$.
Conversely, if $A$ is ancestor of $B$ in $\G$, then this implies the existence of a directed path $\pi = \seq{A = X_0 \tea (X_1 \tea .. \tea X_{n-1}) \tea X_n = B}$ in $\G$. By Definition 8-(i), each edge $X_i \mem X_{i+1}$ along $\pi$ in $\G$ is also present in $\M$, and again by 8-(ii) is of the form $X_i \tem X_{i+1}$. Concatenating them creates the required path in $\M$ which proves the Lemma.
\end{proof}

\begin{lemma} \label{lemCMAG_13}
In a CMAG $\M$, if $\pi = X_1 \mem .. \mem X_n$ is a path in $\M$, then there is a subsequence of the $X_i$'s that forms an uncovered path between $X_1$ and $X_n$ in $\M$.
Similarly, if $\pi$ is an \textit{ancestral path} $X_1 \tem .. \tem X_n$ from $X_1$ to $X_n$, then there is a subsequence of the $X_i$'s that forms an uncovered ancestral path from $X_1$ to $X_n$ in $\M$.
\end{lemma}
\begin{proof}
Follows directly from Lemma 4 in combination with Lemma 13 in \citep{Richardson1996c_DCCS}.
\end{proof}

Now the proofs for some results in the main article.

\textbf{Lemma 1} \textit{For a directed graph $\G$ and corresponding CMAG $\M$, there is a u-structure $\seq{X,Z,Z',Y}$ in $\M$ iff there is an uncovered itinerary $\pi = \seq{X,Z,U,..,U',Z',Y}$ in $\G$, possibly with $Z=U'$ or $U = U'$, where $\seq{X,Z,U}$ and $\seq{U',Z',Y}$ are a pair of m.e.\ conductors w.r.t.\ the uncovered itinerary $\pi$ in $\G$.}
\begin{proof}
By Definition 9, a \textit{u}-structure $\seq{X,Z,Z',Y}$ implies the existence of an uncovered path $\pi = \seq{X,Z,U_1,..,U_k,Z',Y}$ (possibly with $U_1 = U_k$ or $U_1 = Z', U_k = Z$) between nonadjacent $X$ and $Y$ in $\M$, corresponding to an uncovered itinerary in $\G$ where all nodes $\{Z,Z',U_1,..,U_k\}$ are ancestors of each other, but not of $X$ or $Y$, which implies $\seq{X,Z,U_1}$ and $\seq{U_k,Z',Y}$ are a pair of m.e.\ conductors w.r.t.\ the uncovered itinerary $\pi$ in $\G$.

Conversely, if $\seq{X,Z,U}$ and $\seq{U',Z',Y}$ are a pair of m.e.\ conductors w.r.t.\ an uncovered itinerary $\pi = \seq{X,Z,U,..,U',Z',Y}$ in $\G$, then $\pi$ is a also an uncovered path $\seq{X,Z,..,Z',Y}$ in $\M$, where all intermediate nodes are ancestor of each other, as $\{Z,U,..,U',Z'\} \subset SCC(Z)$, but not ancestors of $X$ or $Y$, and so $X \tea Z$ and $Z' \aet Y$in $\M$, which by Definition 9 implies $\seq{X,Z,Z',Y}$  is a \textit{u}-structure.
\end{proof}

\textbf{Lemma 2} \textit{In a CMAG $\M$, a pair of nodes $\seq{X,Z}$ is part of a u-structure $\seq{X,Z,Z',Y}$ with a node $Y \in \bfY \subseteq pa(SCC(Z)) \setminus adj(\{X,Z\})$, iff $X \in pa(Z)$, and $X$ and $Y$ are connected in the subgraph over $((SCC(Z) \setminus adj(X)) \cup \{X,Z\} \cup \bfY$.}
\begin{proof}
The given implies the existence of some path from $X$, via adjacent nodes in the undirected part of the subgraph, to some node from $\bfY$. Let $Y$ be the first 
node from $\bfY$ encountered along this path, then $\seq{X,Z_1,..,Z_k,Y}$ is a path over distinct nodes where all $Z_i \in SCC(Z)$ are ancestors of each other, but not of $X$ or $Y$.
If the path $\seq{X,Z_1,..,Z_k,Y}$ is not uncovered, then by Lemma \ref{lemCMAG_13} some subsequence  $\seq{X,U_1,..,U_m,Y}$ with $\{U_1,..,U_m\} \subset \{Z_1,..,Z_k\}$ can be chosen so that $\seq{X,U_1,..,U_m,Y}$ is an uncovered path consisting solely of nodes in the subgraph.  
Furthermore, as all nodes adjacent to $X$ in $\M$ are excluded from this subgraph with the exception of $Z$, it means that $Z = U_1 = Z_1$. 
We also know that $m \geq 2$, as all $Y \in \bfY$ were taken not to be adjacent to $Z$, so $Z' = U_m \neq Z$.

Finally, as all $U_i \in SCC(Z)$ are ancestors of each other, but not of $X$ or $Y$, 
in accordance with Definition 9, $\seq{X,Z,Z',Y}$ is a \textit{u}-structure.
\end{proof}

\subsection{Proof of Theorem 1}
In the proof of Theorem 1 we use the following (straightforward) implication:

\begin{lemma} \label{lemVtriple}
In a CMAG $\M$, a virtual collider triple $\seq{A,B,C}$ uniquely corresponds to either:
\begin{enumerate}
\item a virtual \textit{v}-structure $\seq{A,B,C}$, or 
\item a \textit{u}-structure $\seq{A,B,B',C}$, or 
\item a \textit{u}-structure $\seq{A,B',B,C}$, 
\end{enumerate} 
where for the latter two the complementary triple $\seq{A,B',C}$ is also a virtual collider triple.
\end{lemma}
\begin{proof}
If virtual collider triple $\seq{A,B,C}$ corresponds to a virtual \textit{v}-structure, then it cannot be part of a \textit{u}-structure $\seq{A,B,B',C}$ or $\seq{A,B',B,C}$, as that would imply the path from $A$ to $C$ via $B$ is not uncovered, contrary Definition 9. Similarly, if virtual collider triple $\seq{A,B,C}$ corresponds to a \textit{u}-structure $\seq{A,B,B',C}$, then it cannot also correspond to a \textit{u}-structure $\seq{A,B',B,C}$, as the combination would imply the presence of edges $A \tea B$ and $B \aet C$ in $\M$, which again would contradict the fact that the path $\seq{A,B,..,B',C}$ in $\M$ is uncovered.
By Definition 10, in both cases the \textit{u}-structure would imply that the complementary $\seq{A,B',C}$ also satisfies the definition of a virtual collider triple.
\end{proof}

We are now ready to prove the new ancestral CET:

\textbf{Theorem 1} \textit{Two directed graphs $\G_1$ and $\G_2$, corresponding to CMAGs $\M_1$ and $\M_2$, are Markov (\textit{d}-separation) equivalent iff
\begin{enumerate}  
\item[(i)] $\M_1$ and $\M_2$ have the same skeleton,
\item[(ii)] $\M_1$ and $\M_2$ have the same \textit{v}-structures,
\item[(iii)] $\M_1$ and $\M_2$ have the same virtual collider triples,
\item[(iv)] if $\seq{A,B,C}$ is a virtual collider triple, and $\seq{A,D,C}$ a virtual \textit{v}-structure, then $B$ is an ancestor of $D$ in $\M_1$ iff $B$ is an ancestor of $D$ in $\M_2$.
\end{enumerate}}
\begin{proof}
We show that in terms of the CPAG the first 3 rules are equivalent to the first 4 rules in the original CET, and that the last rule is sound and implies the last two rules in the original CET, which means the combined set of rules is sound and sufficient to ensure Markov equivalence.

(i) By Lemma  3, two nodes in a CMAG $\M$ are adjacent if and only if they are (virtually) adjacent in the underlying graph $\G$, and so rule (i) is equivalent between the two CETs.

(ii)+(iii) By definitions 4 and 5 and rule (i), an unshielded triple $\seq{A,B,C}$ in a CPAG is either an unshielded conductor, an unshielded perfect nonconductor, or an unshielded imperfect nonconductor in $\G$. Therefore (ii).a+(ii).b in the original CET are equivalent to `have the same unshielded perfect and imperfect nonconductors' (as the remaining unshielded triples then all must correspond to unshielded conductors). An unshielded perfect nonconductor in $\G$ is a \textit{v}-structure in the CMAG $\M$, and by Definition 8 the subset of unshielded \textit{imperfect} nonconductors is equivalent to the set of virtual \textit{v}-structures. By Lemma \ref{lemVtriple}, a virtual collider triple $\seq{A,B,C}$ is either a virtual \textit{v}-structure, or part of a \textit{u}-structure $\seq{A,B,B',C}$ or $\seq{A,B',B,C}$, for which, by Definition 10, the complementary $\seq{A,B',C}$ is also a virtual collider triple. By Lemma 1, that means that, depending on the skeleton from rule (i), either $\seq{A,B,U}$ and $\seq{U',B',C}$ are a pair of m.e.\ conductors w.r.t. uncovered itinerary $\seq{A,B,U,..,U',B',C}$, or $\seq{A,B',U'}$ and $\seq{U,B,C}$ are a pair of m.e.\ conductors w.r.t. uncovered itinerary $\seq{A,B',U',..,U,B,C}$. The latter all follow from rule (iii) in the original CET, and therefore rules (ii) + (iii) combined are equivalent to rules (ii).a + (ii).b + (iii) in the original CET.

(iv) If virtual collider triple $\seq{A,B,C}$ is a virtual \textit{v}-structure, then, by Definition 8, rule (iv) is equivalent to the original CET rule (iv), and therefore sound. If $\seq{A,B,C}$ is part of a \textit{u}-structure, then, by Lemma 1, rule (iv) is equivalent to the original CET rule (v), and therefore also sound.
By Lemma \ref{lemVtriple}, these are the only two possibilities for virtual collider triple $\seq{A,B,C}$, and so rule (iv) is sound.

In conclusion, all rules in Theorem 1 are sound, and imply the rules in the original CET. Therefore, Theorem 1 suffices to establish \textit{d}-separation equivalence, which in turn, under the assumed global directed Markov property, ensures Markov equivalence between two graphs $\G_1$ and $\G_2$.
\end{proof}


\subsection{Proof of Theorem 2}
First a result on invariant ancestral paths from virtual collider triples in CMAGs.

\begin{lemma} \label{lemCMAG_InvAncPath}
 Let CMAGs $\M_1$ and $\M_2$ agree on CET(i)-(iii) in Theorem 1. Let $\seq{A,B,C}$ be a virtual collider triple in $\M_1$, and let there be an uncovered ancestral path \mbox{$\pi = B \tem X_1 \tem .. \tem X_k$} in $\M_1$, so that $B$ is an ancestor of $X_k$ in $\M_1$.
Assume there are no (virtual) \textit{v}-structures $\seq{A,X_i,C}$ along the path $\pi$ in $\M_1$. 
 Then $\seq{A,B,C}$ is also a virtual collider triple in $\M_2$, \mbox{$B \tem X_1 \tem .. \tem X_k$} is an uncovered ancestral path in $\M_2$, and in particular, $B$ is ancestor of $X_k$ in $\M_2$ as well.
\end{lemma}
\begin{proof}
If $\M_1$ and $\M_2$ agree on CET(i)-(iii), then both have the same skeleton, \textit{v}-structures, and virtual collider triples. In particular, it implies that $\seq{A,B,C}$ is also a virtual collider triple in $\M_2$, and that the uncovered path $\pi$ in $\M_1$ corresponds to an uncovered path $\pi_2 = B \mem X_1 \mem .. \mem X_k$ in $\M_2$ as well. Remains to show this path is also ancestral in $\M_2$.

By contradiction: assume the path $\pi_2$ is \textit{not} ancestral in $\M_2$, and let $X_j$ be the first node along $\pi_2$ (starting from $X_0 = B$) that has $X_j \aet X_{j+1}$ in $\M_2$. No (virtual) \textit{v}-structures $\seq{A,X_{i \geq 1},C}$ implies that every node $X_i$ along $\pi_2$ has at most one edge to $A$ or $C$ (i.e.\ not to both), so assume no edge between $A$ and $X_{j+1}$. 

Now virtual collider triple $\seq{A,B,C}$ implies there is an uncovered path $A = X_{-(m+1)} \tea X_{-m} \tet .. \tet X_0 = B$ (possibly $X_{-m} = B$) in $\M_2$. We concatenate this path and $\pi_2$ to obtain $\pi^+ = X_{-(m+1)} \tea X_{-m} \tet .. \tet X_0 \tem X_1 \tem .. \tem X_j \aet X_{j+1}$ in $\M_2$.
Let $X_{i \leq j}$ be the node along $\pi^+$, closest to $X_j$, that has an arc $X_{i-1} \tea X_i$ in $\M_2$. Note there is no edge between $X_{i-1}$ and $X_{j+1}$: for $0 < i \leq j$ because $\pi_2$ was uncovered, and for $i = -m$ because we assumed no edge between $A$ and $X_{j+1}$, and for the remaining $-m < i \leq 0$ the edges are all undirected. 

Then there is a subpath $X_{i-1} \tea X_i \tet .. \tet X_j \aet X_{j+1}$ (possibly $X_i = X_j$) in $\M_2$ where all intermediate nodes are ancestors of each other, but not of $X_{i-1}$ or $X_{j+1}$, and so by Lemma \ref{lemCMAG_13} there is also an uncovered path $X_{i-1} \tea X_p \tet .. \tet X_q \aet X_{j+1}$ in $\M_2$. This path would either correspond to a (virtual) \textit{v}-structure (if $p=q$), or a \textit{u}-structure (if $p<q$) in $\M_2$, and so, if $\M_1$ and $\M_2$ agree on CET(i)-(iii), would also appear as $X_q \aet X_{j+1}$ in $\M_1$. But that would imply $X_q$ is not an ancestor of $X_{j+1}$, in contradiction with the ancestral path from $A$ via $X_q$ to $X_{j+1}$ in $\M_1$.

Therefore, there can be no edge $X_j \aet X_{j+1}$ along $\pi_2$ in $\M_2$, and so $\pi_2 = B \tem X_1 \tem .. \tem X_k$ is also an uncovered ancestral path in $\M_2$, and in particular, $B \in an(X_k)$ as well.
\end{proof}

We can apply the same approach to invariant descendants of \textit{v}-structures.
\begin{lemma} \label{lemCMAG_InvDescVstruct}
 Let CMAGs $\M_1$ and $\M_2$ agree on CET(i)-(iii) in Theorem 1. Let $X_k$ be a descendant of a node $X_0$ in a \textit{v}-structure $\seq{A,X_0,C}$ in $\M_1$. Then $X_k$ is also a descendant of some node $X_j$ (possibly $X_j = X_0$) in a \textit{v}-structure $\seq{A,X_j,C}$ in $\M_2$. 
\end{lemma}
\begin{proof}
If $\M_1$ and $\M_2$ agree on CET(i)-(iii), then both have the same skeleton, \textit{v}-structures, and virtual collider triples. Therefore, if $X_k $ is part of a \textit{v}-structure $\seq{A,X_k,C}$ in $\M_1$, then also in $\M_2$, and so the lemma is trivially true.

If not, then $X_k \in de(X_0)$ in $\M_1$ implies there is an ancestral path from $X_0$ to $X_k$ in $\M_1$, and so by Lemma \ref{lemCMAG_13} also an uncovered ancestral path \mbox{$\pi = X_0 \tem X_1 \tem .. \tem X_k$} in $\M_1$. By CET(i) the path $\pi$ is also an uncovered path in $\M_2$. Furthermore, by Definition 8-(iii), none of nodes $X_i$ along the path can be part of a virtual \textit{v}-structure $\seq{A,X_i,C}$ in $\M_1$, and so also not in $\M_2$.

Let $X_j$ be the node closest to $X_k$ along the path $\pi$ that is part of a \textit{v}-structure $\seq{A,X_j,C}$ (possibly $X_j = X_0$), so there are no other (virtual) \textit{v}-structures $\seq{A,X_{i>j},C}$ along the path in $\M_1$.

We can now apply the same argument as in the proof of Lemma \ref{lemCMAG_InvAncPath} to conclude that $X_j \tem .. \tem X_k$ is also an uncovered ancestral path in $\M_2$, and therefore indeed $X_k$ is also a descendant of a \textit{v}-structure $\seq{A,X_j,C}$ in $\M_2$.
\end{proof}

Lemma \ref{lemCMAG_13}  has the following implication, analogous to the Claim in Case 5 in Theorem 2 in  \citep[p.38]{Richardson1996c_DCCS}, showing that CET (iv) is only needed to distinguish edges at virtual \textit{v}-structures.
\begin{lemma} \label{lemCMAG_Thm2}
Let CMAGs $\M_1$ and $\M_2$ agree on CET(i)-(iii) in Theorem 1. 
Let $\seq{A,B,C}$ be a virtual collider triple and $\seq{A,D,C}$ be a virtual \textit{v}-structure in $\M_1$, with $B$ an ancestor of $D$ in $\M_1$, but not in $\M_2$, i.e.\ they do not agree on CET (iv) and so $\M_1$ and $\M_2$ are not Markov equivalent. 
Then they also differ on CET (iv) for some virtual collider triple $\seq{A,B',C}$ (possibly $B' = B$), and virtual \textit{v}-structure $\seq{A,D',C}$ (possibly $D' = D$), such that $B' \in an(D')$ in $\M_1$, but $B' \notin an(D')$ in $\M_2$, where there is an uncovered ancestral path $B' \tem X_1 \tem .. \tem X_k \tem D'$ in $\M_1$ (possibly $B' = X_k$), corresponding to an uncovered path $B' \tem X_1 \tem .. \tem X_k \aet D'$ in $\M_2$.
\end{lemma}
\begin{proof}
In words: if CET (iv) is needed to distinguish between two CMAGs that are not Markov equivalent, but agree on CET (i)-(iii), then they differ on an edge $X_k \mem D'$ at a virtual \textit{v}-structure $\seq{A,D',C}$ along an uncovered ancestral path (except for edge $X_k \mem D$) from another virtual collider triple $\seq{A,B',C}$ that is an ancestor of $X_k$ in both $\M_1$ and $\M_2$.

As $B$ is ancestor of $D$ in $\M_1$ there is an ancestral path from $B$ to $D$ in $\M_1$, and so by Lemma \ref{lemCMAG_13} also an uncovered ancestral path from $B$ to $D$. 

Starting from $B$, let $D'$ be the first node in a virtual \textit{v}-structure $\seq{A,D',C}$ along this uncovered ancestral path in $\M_1$, such that $B \notin an(D')$ in $\M_2$ (possibly $D' = D$). Let $B'$ be the node in a virtual \textit{v}-structure  $\seq{A,B',C}$ closest to $D'$ on the path from $B$, or $B' = B$ if no such triple exists. 

Then there is an uncovered ancestral path $\pi = B' \tem X_1 \tem .. \tem X_k \tem D'$ in $\M_1$, and so $B' \in an(D')$ in $\M_1$. However $B' \notin an(D)$ in $\M_2$: by construction $B$ is an ancestor of $B'$ in $\M_2$ (otherwise $B'$ would satisfy the criterion for $D'$ but be closer to $B$, contrary the assumed), so if $B' \in an(D')$ in $\M_2$, then $B$ would also be ancestor of $D'$ in $\M_2$, again contrary the assumed.

At least one node along the path $\pi$ is not an ancestor of its successor along the path in $\M_2$, otherwise $B'$ would be ancestor of $D'$ in $\M_2$ as well.
Let $X_j$ be the first such node along the path starting from $B'$, such that $B' \tem X_1 \tem .. \tem X_j \aet X_{j+1} \mem .. \mem X_k \mem D'$ in $\M_2$. By definition, no node $X_i$ along this path appears in a \textit{v}-structure $\seq{A,X_i,C}$ in $\M_1$, otherwise $\seq{A,D',C}$ would not be a virtual \textit{v}-structure. By construction, there are also no virtual \textit{v}-structures $\seq{A,X_i,C}$ on this path between $B'$ and $D'$ in $\M_1$.

We now show this implies $X_{j+1} = D'$. By contradiction, suppose $X_{j+1} \neq D'$. Then there exists an uncovered ancestral path $B' \tem .. \tem X_j \tem X_{j+1}$ in $\M_1$, where no node $X_i$ along this path (except perhaps $B'$) is part of a virtual \textit{v}-structure $\seq{A,X_i,C}$. Therefore by Lemma \ref{lemCMAG_InvAncPath}, then $B' \tem .. \tem X_j \tem X_{j+1}$ in $\M_2$, in contradiction with the assumed $X_j \aet X_{j+1}$. Therefore $X_{j+1} = D'$, and $X_j = X_k$,  and so $B' \tem .. \tem X_k \aet D'$ is an uncovered path in $\M_2$. It also implies that $B' \in an(X_k)$ in both $\M_1$ and $\M_2$. 
\end{proof}

It means that if two CMAGs $\M_1$ and $\M_2$ are only different on CET rule (iv), then they (also) differ on rule (iv) between two nodes connected by a very specific path configuration, which we will use in the subsequent \textit{Graph-to-CPAG} algorithm. 

Next we show that for the case $\M_2$ with $B \notin an(D)$ above, the final edge $X_k \aet D'$ is an invariant edge in any CMAG that is Markov equivalent to $\M_2$. We also show that in order to identify this invariant edge, we do not need to find all possible pairs of nodes $B$ and $D$ that satisfy the conditions in Lemma \ref{lemCMAG_Thm2} above, but only the existence of \textit{some} pair that do.

\begin{lemma} \label{lemExistW-D}
In CMAG $\M_1$, let $\seq{A,D,C}$ be a virtual \textit{v}-structure, with $X_k \aet D$ in $\M_1$, with $X_k$ is not a descendant of some node $X'$ in a \textit{v}-structure $\seq{A,X',C}$.
Assume there exists a virtual collider triple $\seq{A,B,C}$ in $\M_1$ such that $B$ is not an ancestor of $D$, $B$ is an ancestor of $X_k$, and there exists an uncovered path $\pi = B  \tem X_1 \tem .. \tem X_k \aet D$ in $\M_1$. Then there exists a node $B'$ in a virtual collider triple $\seq{A,B',C}$ in $\M_1$ (possibly $B' = B$), such that for any CMAG $\M_2$ that is Markov equivalent to $\M_1$, in both $\M_1$ and $\M_2$ it holds that: 1)  $B'$ is not an ancestor of $D$, 2) $B'$ is an ancestor of $X_k$, 3) $X_k$ is not a descendant of some node $X'$ in a  \textit{v}-structure $\seq{A,X',C}$, and 4) there exists an uncovered path $\pi = B ' \tem X_j \tem .. \tem X_k \aet D$, and in particular, then $X_k \aet D$ in $\M_2$ as well.
\end{lemma}
\begin{proof} 
The given implies that $\M_1$ and $\M_2$ agree on CET rules (i)-(iv). By rules (i)-(iii), both have the same skeleton, uncovered paths, \textit{v}-structures, and virtual collider triples. Then if $X_k$ is not a descendant of some node $X'$ in a \textit{v}-structure $\seq{A,X',C}$ in $\M_1$, then by Lemma \ref{lemCMAG_InvDescVstruct} neither in $\M_2$, so (3) holds.

Let $B'$ be the node closest to $D$ along the path $\pi$ in $\M_1$ that is part of a virtual \textit{v}-structure $\seq{A,B',C}$ (possibly $B' = B$ or $B' = X_k$), or $B' = B$ if no such virtual \textit{v}-structure exists. Then the subpath $\pi' = B' \tem X_j \tem .. \tem X_k \aet D$ is an uncovered path in $\M_1$, and by construction there are no other \textit{v}-structures or virtual \textit{v}-structures $\seq{A,X_i,C}$ along the path $\pi'$ between $B'$ and $D$.
Then by Lemma \ref{lemCMAG_InvAncPath} the subpath $B' \tem X_j \tem .. \tem X_k$ is also an uncovered ancestral path in $\M_2$, and by CET(i), $\pi'$ itself is also an uncovered path in $\M_2$, and therefore $B' \tem X_j \tem .. \tem X_k \mem D$ is an uncovered path in $\M_2$, which implies (2) holds as well. 

Node $B'$ is not an ancestor of $D$ in $\M_1$ (otherwise $B$ as ancestor of $B'$ would be as well, contrary the given). For Markov equivalent graph $\M_2$, CET rule (iv) on virtual collider triple $\seq{A,B',C}$ and virtual \textit{v}-structure $\seq{A,D,C}$ then implies $B'$ is not an ancestor of $D$ in $\M_2$ either, which ensures (1).

But then by contradiction, if $X_k \tea D$ or $X_k \tet D$ in $\M_2$, then $\pi'$ would be an ancestral path from $B'$ to $D$, which would imply $B'$ is an ancestor of $D$ in $\M_2$, contrary the given. Therefore $X_k \aet D$ in $\M_2$ as well, which proves (4).
\end{proof}
The Lemma above implies that if the triggering conditions in the second clause of the `if' statement in part 2 of Algorithm 2 apply to a CMAG $\M_1$, then the same conclusion $X_k \aet D$ will (also) trigger on a pattern that is \textit{identical} between all CMAGs Markov equivalent to $\M_1$.

Note that for $B \in an(D)$ in $\M_1$, the final edge $X_k \tem D'$ is \textit{not} necessarily an invariant edge in every CMAG that is Markov equivalent to $\M_1$, as an edge $X_k \aet D'$ does not preclude the existence of another path $\pi'$ along which $B'$ is an ancestor of $D'$. Therefore, the invariance in Lemma \ref{lemExistW-D} only applies to the explicit case `\textit{$B \notin an(D)$}'.

We can also show that when Lemma \ref{lemExistW-D} is needed to orient an invariant $X_k \aet D$, then $X_k$ is part of a cycle.
\textbf{Corollary 1 } In a CMAG $\M$, let $\seq{A,B,C}$ be a virtual collider triple, and $\seq{A,D,C}$ be a virtual \textit{v}-structure, with $B \notin an(D)$. Let $\pi = B  \tem X_1 \tem .. \tem X_k \aet D$ be an uncovered path in $\M$ (possibly $B = X_k$), with $X_k$ not a descendant of a node $X'$ in a \textit{v}-structure $\seq{A,X',C}$.
Then, if $X_k \aet D$ is not already implied by CET rules (i)-(iii), then $X_k$ is part of a cycle in $\M$.
\begin{proof}
If $X_k$ is part of a virtual collider triple $\seq{A,X_k,C}$, then by definition $X_k$ is part of a cycle, and so then the claim holds.

If $X_k$ is \textit{not} part of a virtual collider triple $\seq{A,X_k,C}$, then if $X_{k-1} \tea X_k$ in $\M$, then $X_{k-1} \tea X_k \aet D$ would be a (virtual) \textit{v}-structure, and so already be implied by CET (i)-(iii). Therefore, if CET (iv) is needed to orient $X_k \aet D$, then $X_{k-1} \tem X_k \aet D$ must be an unshielded noncollider, and so $X_{k-1} \tet X_k$, which implies $X_{k-1}$ and $X_k$ are part of a cycle.
\end{proof}
We use this in the implementation to quickly filter out candidate nodes $W$ in part 2 of Algorithm 2 that are not part of a cycle ($|SCC(W)| = 1$), in order to avoid unnecessary path searches.

Finally, we show that all edges between two virtual collider nodes in a CMAG are invariant.

\begin{lemma} \label{lemV-V_inv}
If two CMAGs $\M_1$ and $\M_2$ are Markov equivalent, then if $\seq{A,B,C}$ and $\seq{A,D,C}$ are virtual collider triples with $B$ and $D$ adjacent, then $B$ is ancestor of $D$ in $\M_1$ iff $B$ is ancestor of $D$ in $\M_2$.
\end{lemma}
\begin{proof}
We consider three cases: 1) both virtual collider triples correspond to virtual \textit{v}-structures, 2) one virtual collider triple corresponds to a virtual \textit{v}-structure, and the other is part of a \textit{u}-structure, or 3) both are part of a \textit{u}-structure. Below we will tackle each of these cases in turn:

\textit{Case 1}: this is equivalent to rule (iv) of the original CET.

\textit{Case 2}: let $\seq{A,B,C}$ be the virtual \textit{v}-structure, and $\seq{A,D,C}$ be part of a \textit{u}-structure $\seq{A,D,D',C}$. Note this implies there is no edge between $C$ and $D$.
Then if $B \tea D$ in $\M_1$, i.e.\ $B$ is NOT a descendant of $D$, then this satisfies rule (v) of the original CET, meaning $B$ is also not a descendant of $D$ in $\M_2$, and so $B \tea D$ in $\M_2$ as well. 
If $B \aet D$ in $\M_1$, then $C \tea B \aet D$ would be a (virtual) \textit{v}-structure, and be invariant by CET rule (ii)/(iii). 
The only remaining possibility is $D \tet B$ in $\M_1$, which by symmetry then must also apply to $\M_2$. 

\textit{Case 3}: now both $\seq{A,B,C}$ and $\seq{A,D,C}$ are part of a \textit{u}-structure, but neither are virtual \textit{v}-structures. 

Consider \mbox{$B \tea D$} in $\M_1$. As $A$ and $C$ cannot both have an edge to $B$ (for then $\seq{A,B,C}$ would be a virtual \textit{v}-structure), assume there is no edge between $B$ and $C$. 
Then if virtual collider triple $\seq{A,D,C}$ corresponds to a \textit{u}-structure $\seq{A,D',D,C}$, then $B \tea D \aet C$ corresponds to a (virtual) \textit{v}-structure, and would be invariant by CET rule (ii)/(iii). 
If virtual collider triple $\seq{A,D,C}$ corresponds to a \textit{u}-structure $\seq{A,D,D',C}$, then $B \tea D \tet .. \tet D'\aet C$ would be a path in $\M_1$. Let $D^*$ be the node closest to $C$ along the path with an edge to $B$. If $D^* = D'$, then $B \tea D^* \aet C$ would be a invariant (virtual) \textit{v}-structure. If $D^* \neq D'$ then $B \tea D^* \tet .. \tet D' \aet C$ would be an invariant \textit{u}-structure in $\M_1$, and so then $B \tea D^*$ in $\M_2$. Given that $D$ and $D^*$ are both ancestors/descendants of each other, it follows that then $B \tea D$ in $\M_2$ as well.
In all cases this implies edge $B \tea D$ would be invariant by CET rule (ii)/(iii), and so appear as $B \tea D$ in $\M_2$.

For $B \aet D$ in $\M_1$ we can repeat the previous argument with the roles of $B$ and $D$ reversed, which implies $B \aet D$ in $\M_2$ as well.
That leaves $B \tet D$ in $\M_1$ as the only remaining option, and so necessarily must appear as $B \tet D$ in $\M_2$ as well.
\end{proof}

We can now prove that the Graph-to-CPAG algorithm in the main article is sound and \textit{d}-separation complete, meaning that the output CPAG can be used to establish Markov equivalence between cyclic directed graphs.

\textbf{Theorem 2} \textit{For two different directed graphs $\G_1$ and $\G_2$, let $\cP_1$ and $\cP_2$ be the corresponding CPAGs output by the Graph-to-CPAG algorithm. Then $\G_1$ is Markov (\textit{d}-separation) equivalent to $\G_2$ iff $\cP_1 = \cP_2$.}
\begin{proof}
We cover three aspects of the claim: 1) soundness of the output PAG, 2) \textit{d}-separation completeness of the output PAG (which implies it is a CPAG), and 3) equality between CPAGs if and only if they are Markov equivalent.

1) Soundness of the algorithm follows from Theorem 1, in combination with the fact that: a) each orientation in part 1 of the algorithm has a direct match to an invariant feature implied by the CET rules (i)-(iii), b) all orientations between adjacent triples from the first `if' clause in part 2 of the algorithm are sound by Lemma \ref{lemV-V_inv}, and c) the remaining orientations between nonadjacent triples implied by CET rule (iv) from the second `if' clause in part 2 are sound by Lemma \ref{lemExistW-D}. 
Therefore all orientations correspond to invariant features in the Markov equivalence class of the input graph $\G$, which guarantees the output is a valid PAG. 

2) \textit{d}-separation completeness follows from the fact that if two graphs $\G_1$ and $\G_2$ are not Markov equivalent, then the algorithm will make at least one different orientation in the corresponding output PAGs $\cP_1$ and $\cP_2$ (which therefore qualify as CPAGs).

Part 1 of the algorithm captures the entire skeleton, all \textit{v}-structures, all virtual \textit{v}-structures, and all invariant edges from \textit{u}-structures into a cycle. Therefore, if CMAGs $\M_1$ and $\M_2$ corresponding to graphs $\G_1$ resp.\ $\G_2$ differ in \textit{any} feature corresponding to CET rules (i)-(iii), then this will lead to at least one different edge/orientation between the output PAGs $\cP_1$ and $\cP_2$. 

If $\M_1$ and $\M_2$ agree on CET (i)-(iii), but are not Markov equivalent (i.e.\ they disagree only on CET (iv)), then by Lemma \ref{lemCMAG_Thm2} they differ on (at least) one $X_k \mem D$ along an uncovered path $B \tem X_1 \tem .. \tem X_k \mem D$ between some virtual collider triple $\seq{A,B,C}$ and virtual \textit{v}-structure $\seq{A,D,C}$, for which $B \in an(D)$ in $\M_1$, but $B \notin an(D)$ in $\M_2$.

In part 2 of the algorithm, if $B \notin an(D)$ in the CMAG $\M$ corresponding to input graph $\G$, then \textit{all} such uncovered paths (or edge) between $B$ and $D$ will obtain an orientation $X_k \aet D$ (possibly $X_k = B$) in the output PAG $\cP$. If $B \in an(D)$ in the CMAG $\M$, then \textit{at least one }of these paths must have $X_k \tem D$ in $\M$ (again possibly $X_k = B$), otherwise there would be no ancestral path from $B$ to $D$ in $\M$, contrary the assumption that $B \in an(D)$. By the soundness of the algorithm, this edge will obtain either $X_k \tea D$, $X_k \tet D$, or $X_k \cec D$ in the output PAG $\cP$ (again possibly with $X_k = B$). 

Therefore if two CMAGs $\M_1$ and $\M_2$ are not Markov equivalent, then there is at least one edge or orientation different between the corresponding output PAGs $\cP_1$ and $\cP_2$. Therefore the output PAG $\cP$ uniquely identifies the Markov equivalence class of input graph $\G$, which also implies the output PAG $\cP$ is indeed a CPAG.

3) equality: the `only if' part follows from the \textit{d}-separation completeness of the algorithm, above. Remains to show that when two graphs $\G_1$ and $\G_2$ are Markov equivalent then the corresponding output CPAGs $\cP_1$ and $\cP_2$ are also \textit{identical}. For an input graph $\G_1$, part 1 of the algorithm exhaustively searches for all and only the edges (skeleton), \textit{v}-structures, and virtual collider triples implied by CET (i)-(iii). Any graph $\G_2$ that is Markov equivalent to $\G_1$ must have the same skeleton, \textit{v}-structures, and virtual collider triples, and therefore the two (intermediate) PAGs $\cP_1$ and $\cP_2$ must be identical after part 1.

For the remaining orientations in part 2 of the algorithm, note that all required elements: virtual \textit{v}-structures, virtual collider triples, and skeleton are implied by CET rules (i)-(iii), and therefore identical between the corresponding CMAGs $\M_1$ and $\M_2$. By Lemma \ref{lemExistW-D}, if an orientation in part 2 is triggered for some CMAG $\M_1$, then it would also trigger on a pattern in $\M_1$ that is \textit{invariant} between all Markov equivalent CMAGs. 

Therefore, for Markov equivalent $\M_1$ and $\M_2$, any orientation for $\cP_1$ in part 2 of the algorithm will also be oriented identically in $\cP_2$ and v.v. Combined with the fact that $\cP_1$ and $\cP_2$ are identical after part 1 of the algorithm, this proves the `if' part of the Theorem.

As a result, for graphs $\G_1$ and $\G_2$, the corresponding output $\cP_1$ and $\cP_2$ of the Graph-to-CPAG algorithm is identical iff $\G_1$ and $\G_2$ are Markov equivalent.
\end{proof}


\section{MARKOV PROPERTIES FOR STRUCTURAL CAUSAL MODELS}\label{sec:scm}

We state here some of the key definitions and results in the theory of Structural Causal Models (SCMs).
These models, also known as Structural Equation Models (SEMs), were introduced a century ago by \citet{Wright1921} and popularized in AI by \citet{Pearl2009}.
We follow here the treatment of \citet{Bongers++_AOS_21}, as it deals with cycles in a mathematically rigorous way.

\begin{dfn}[SCM]\label{def:SCM}
A Structural Causal Model (SCM) is a tuple $M = \langle \bfV, \bfW, \dom{\bfV}, \dom{\bfW}, \bff, P_M \rangle$ of:
\begin{enumerate}
\item finite disjoint index sets $\bfV, \bfW$ for the endogenous and exogenous variables in the model, respectively;
\item a product of standard measurable spaces $\dom{\bfV} = \prod_{v \in \bfV} \dom{v}$, which define the domains of the endogenous variables; 
\item a product of standard measurable spaces $\dom{\bfW} = \prod_{w \in \bfW} \dom{w}$, which define the domains of the exogenous variables;
\item a measurable function $\bff : \dom{\bfV} \times \dom{\bfW} \to \dom{\bfV}$, the \emph{causal mechanism};
\item a product probability measure $P_M = \prod_{w \in \bfW} P_{w}$ on $\dom{\bfW}$, with each $P_w$ a probability measure on $\dom{w}$, specifying the \emph{exogenous distribution}.
\end{enumerate}
\end{dfn}
The causal structure of the SCM is encoded by the dependences of the components of $\bff$ on the variables in the model. 
This is formalized by:
\begin{dfn}[Parent]
Let $M$ be an SCM. We call $i \in \bfV \cup \bfW$ a parent of $k \in \bfV$ if and only if there does not exist a measurable
function $\tilde f_k : \dom{\bfV\setminus \{i\}} \times \dom{\bfW\setminus\{i\}} \to \dom{k}$ such that 
for $P_M$-almost every $\bfw \in \dom{\bfW}$, for all $\bfv \in \dom{\bfV}$, 
  $$v_k = \tilde f_k(\bfv_{\setminus i},\bfw_{\setminus i}) \iff v_k = f_k(\bfv,\bfw).$$
\end{dfn}
Intuitively, this means that the $k$'th component of $\bff$ does depend on the $i$'th variable.
This definition allows us to define the directed mixed graph (DMG) associated to an SCM:
\begin{dfn}[Graph]
Let $M$ be an SCM. The \emph{graph} of $M$, denoted $\G(M)$, is defined as the directed mixed graph
with nodes $\bfV$, directed edges $v_1 \to v_2$ iff $v_1$ is a parent of $v_2$ according to $M$, and bidirected edges
$v_1 \aea v_2$ iff there exists $w \in \bfW$ such that $w$ is parent of both $v_1$ and $v_2$ according to $M$.
\end{dfn}
If $\G(M)$ is acyclic, we call the SCM $M$ \emph{acyclic}, otherwise we call the SCM \emph{cyclic}. 
If $\G(M)$ contains no bidirected edges, we call the endogenous variables in the SCM $M$ \emph{causally sufficient}
(which is what we assumed in the present work for simplicity).

SCMs provide an implicit description of their solutions.
\begin{dfn}[Solutions]
A random variable $\rv{} = (\rv{\bfV},\rv{\bfW})$ is called a \emph{solution} of the SCM $M$ if
$\rv{\bfV} = (\rv{v})_{v \in \bfV}$ with $\rv{v} \in \dom{v}$ for all $v \in \bfV$,
$\rv{\bfW} = (\rv{w})_{w \in \bfW}$ with $\rv{w} \in \dom{w}$ for all $w \in \bfW$,
the distribution $\Prb(\rv{\bfW})$ is equal to the exogenous distribution $P_M$, and
the \emph{structural equations}:
$$\rv{v} = f_v(\rv{\bfV}, \rv{\bfW})\quad\text{a.s.}$$
hold for all $v \in \bfV$.
\end{dfn}

For acyclic SCMs, solutions exist and have a unique distribution that is determined by the SCM.
This is not generally the case in cyclic SCMs, as these could have no solution at all, or 
could have multiple solutions with different distributions.
\begin{dfn}[Unique solvability]\label{def:unique_solvability_wrt}
An SCM $M$ is said to be \emph{uniquely solvable w.r.t.\ $\bfO \subseteq \bfV$} if there exists 
a measurable mapping $\bfg_{\bfO} : \dom{\pa_{\G(M)}(\bfO)\setminus\bfO} \to \dom{\bfO}$ 
such that for $P_M$-almost every $\bfw \in \dom{\bfW}$, for all $\bfv \in \dom{\bfV}$:
  \begin{equation*}\begin{split}
    &\bfv_{\bfO} = \bfg_{\bfO}(\bfv_{(\pa_{\G(M)}(\bfO)\setminus\bfO)\cap\bfV}, \bfw_{\pa_{\G(M)}(\bfO)\cap\bfW}) \\
    &\quad\iff\quad \bfv_{\bfO} = \bff_{\bfO}(\bfv,\bfw).
  \end{split}\end{equation*}
\end{dfn}
Loosely speaking: the structural equations for $\bfO$ have an essentially unique solution for $\bfv_{\bfO}$ in terms of the other variables appearing in those equations.
If $M$ is uniquely solvable with respect to $\bfV$ (in particular, this holds if $M$ is acyclic), then it induces a unique \emph{observational distribution} $P_M(\rv{\bfV})$, the push-forward of $P_M$ through $\bfg_{\bfV}$.

One of the key aspects of SCMs---which we do not discuss here in detail because we do not make use of it in this work---is their causal semantics, which is defined in terms of interventions.
Instead, we discuss only their probabilistic properties.
In particular, under appropriate assumptions, the graph $\G(M)$ of an SCM $M$ represents conditional independences that its solutions must satisfy.
As shown already by \citet{Spirtes1994,Spirtes1995}, the directed global Markov property does {\it not} hold in general for cyclic SCMs.
\begin{exm}[$d$-separation fails]
Consider the SCM $M = \langle \{A,B,C,D\}, \{5,6,7,8\}, \RN^4, \RN^4, \bff, P_M \rangle$ where
$P_M$ is the standard-normal distribution on $\RN^4$, and the causal mechanism is given by:
  $$\bff(\bfx) = (x_5, x_6, x_A x_D + x_7, x_B x_C + x_8)$$
The graph $\G(M)$ is depicted in Figure 1 (left). 
This SCM is uniquely solvable with respect to its strongly connected components $\{A\}$, $\{B\}$, and $\{C,D\}$. 
One can check that for every solution $\rv{}$ of M, $\rv{A}$ is not independent of $\rv{B}$ given $\{\rv{C}, \rv{D}\}$. 
However, the nodes $A$ and $B$ are $d$-separated given $\{C, D\}$ in $\G(M)$. 
Hence the global directed Markov property does not hold for $M$.
\end{exm}
For more concrete examples of cyclic SCMs, we refer the reader to \citep{Bongers++_AOS_21}.
\citet{Spirtes1994} proved a weaker Markov property in terms of a `collapsed graph', assuming causal sufficiency and densities.
\citet{ForreMooij_1710.08775} found the following formulation in terms of `$\sigma$-separation' that is immediately applicable to the graph of the SCM itself.
\begin{dfn}[Blockable and unblockable noncolliders]
  Let $\G$ be a directed mixed graph and $\pi$ a path in $\G$.
  We call a noncollider on $\pi$ \emph{unblockable} if it is not an end-node and it only has outgoing edges on $\pi$ to nodes in the same strongly connected component of $\G$; otherwise, it is called \emph{blockable}.
\end{dfn}
If $\G$ is acyclic then all noncolliders are blockable.
\begin{dfn}[$\sigma$-separation]
For a triple of node sets $\bfX,\bfY,\bfZ$ in a graph $\G$, we say that $\bfX$ is \emph{$\sigma$-connected} to $\bfY$ given $\bfZ$ iff there is an $X \in \bfX$ and $Y \in \bfY$ such that there is a path $\pi$ between $X$ and $Y$ on which every {\it blockable} noncollider is not in $\bfZ$, and every collider on $\pi$ is an ancestor of $\bfZ$; otherwise $\bfX$ and $\bfY$ are said to be \emph{$\sigma$-separated} given $\bfZ$.
\end{dfn}
Note the small difference with the definition of $d$-connection: $\sigma$-connection only considers the {\it blockable} noncolliders. 
The following general result was shown by \citet{ForreMooij_1710.08775}.
\begin{thm}[$\sigma$-Separation Markov property]\label{thm:sigma_separation}
Let $M$ be an SCM that is uniquely solvable w.r.t.\ each strongly connected component of $\G(M)$.
Then, the observational distribution of $M$ exists and is unique.
Furthermore, for a solution $\rv{}$ of $M$ and
for $\bfA,\bfB,\bfC \subseteq \bfV$:
if $\bfA$ is $\sigma$-separated from $\bfB$ given $\bfC$ in $\G(M)$, then $\rv{\bfA}$ is conditionally independent of $\rv{\bfB}$ given $\rv{\bfC}$.
\end{thm}
\begin{proof}
  See the proof of Theorem A.21 in \citet{Bongers++_AOS_21}. 
\end{proof}
Under certain additional assumptions, one can show the stronger $d$-separation criterion (also known as the global directed Markov property).
\begin{thm}[$d$-Separation Markov property]\label{thm:d_separation}
Let $M$ be an SCM that satisfies one of the following three assumptions:
\begin{enumerate}
  \item $M$ is acyclic;\label{thm:d_separation_acyclic}
  \item \label{thm:d_separation_discrete}
    \begin{itemize}
      \item all endogenous domains $\dom{v}$ for $v \in \bfV$ are discrete, and
      \item $M$ is uniquely solvable w.r.t.\ each ancestral subset $A \subseteq \bfV$ (that is, each subset $A \subseteq \bfV$ such that $\an_{\G(M)}(A) = A$);
  \end{itemize}
\item \label{thm:d_separation_linear}
    \begin{itemize}
      \item $\dom{\bfV} = \RN^{\bfV}$ and $\dom{\bfW} = \RN^{\bfW}$, and
      \item $\bff$ is a linear mapping, and
      \item each $v \in \bfV$ has at least one parent in $\bfW$ according to $M$, and
      \item $P_M$ has a density w.r.t.\ the Lebesgue measure on $\RN^{\bfW}$.
    \end{itemize}
\end{enumerate}
Then, the observational distribution of $M$ exists and is unique.
Furthermore, for a solution $\rv{}$ of $M$ and
for $\bfA,\bfB,\bfC \subseteq \bfV$:
if $\bfA$ is $d$-separated from $\bfB$ given $\bfC$ in $\G(M)$, then $\rv{\bfA}$ is conditionally independent of $\rv{\bfB}$ given $\rv{\bfC}$.
\end{thm}
\begin{proof}
See the proof of Theorem A.7 in \citet{Bongers++_AOS_21}.
The acyclic case is well known. 
The discrete case fixes the erroneous theorem by \citet{PearlD1996}, for which a counterexample was found by \citet{Nea00}, by adding the 
assumption of unique solvability with respect to each ancestral subset, and extends it to allow for bidirected edges in the graph. 
The linear case is an extension of existing results for the linear-Gaussian setting without bidirected edges \citet{Spirtes1994, Spirtes1995, Kos96} to a linear (possibly non-Gaussian) setting with bidirected edges in the graph.
\end{proof}
For this paper, we assume that the global directed Markov property holds with respect to a graph that contains no bidirected edges.
From the above theorem, it follows that this will hold if the data comes from the observational distribution of a causally sufficient SCM that 
falls into either the acyclic case (\ref{thm:d_separation_acyclic}), the discrete case (\ref{thm:d_separation_discrete}), or the linear case (\ref{thm:d_separation_linear}).
Note that these assumptions are sufficient, but not necessary.

\end{document}